\documentclass[]{article}
\usepackage{framed,multirow}

\usepackage{latexsym}

\usepackage{url}
\usepackage{xcolor}

\usepackage{hyperref}

\definecolor{newcolor}{rgb}{.8,.349,.1}

\usepackage{amssymb, amsmath, amsgen, amsfonts, amsthm,bm}

\usepackage{natbib}

\usepackage[utf8]{inputenc}
\usepackage[ruled,vlined]{algorithm2e}

\usepackage{booktabs}
\usepackage{soul}
\definecolor{LightGreen}{HTML}{B6FF00}
\sethlcolor{LightGreen}

\usepackage[caption=false]{subfig}
\usepackage{graphicx}
\graphicspath{{figures/}{figures/autoenc/}{figures/exp23/}{figures/segment/}}

\definecolor{newcolor}{rgb}{.8,.349,.1}

\newtheorem{theorem}{Theorem}[section]

\newcommand{\todo}[1]{}
\renewcommand{\todo}[1]{{\color{black} TODO: {#1}}}

\newcommand{\Var}{\mathrm{Var}}

\newcommand{\MVar}{\mathrm{MVar}}
\newcommand{\nn}{n \times n}

\def\lc{\left\lfloor}   
\def\rc{\right\rfloor}
\title{Adaptive Convolution Kernel for Artificial Neural Networks}
\author{F. Boray Tek, İlker Çam, Deniz Karlı}

\begin{document}

\maketitle

\begin{abstract}
Many deep neural networks are built by using stacked convolutional layers of fixed and single size (often 3$\times$3) kernels. This paper describes a method for training the size of convolutional kernels to provide varying size kernels in a single layer. The method utilizes a differentiable, and therefore backpropagation-trainable Gaussian envelope which can grow or shrink in a base grid. Our experiments compared the proposed adaptive layers to ordinary convolution layers in a simple two-layer network, a deeper residual network, and a U-Net architecture. The results in the popular image classification datasets such as MNIST, MNIST-CLUTTERED, CIFAR-10, Fashion, and “Faces in the Wild” showed that the adaptive kernels can provide statistically significant improvements on ordinary convolution kernels. A segmentation experiment in the Oxford-Pets dataset demonstrated that replacing a single ordinary convolution layer in a U-shaped network with a single 7$\times$7 adaptive layer can improve its learning performance and ability to generalize.
\end{abstract}

\section{Introduction}\label{Int}
Neural network-based pattern recognition is the state-of-the-art approach to solving many visual problems. The most successful solutions are based on stacked convolutional layers \cite{Krizhevsky2017, iandola2016, resnet}. The stacked deep hierarchy allows increasingly complex and discriminative representations (features) which also become easier to classify. Though biological neurons are functionally different, there is firm evidence that biological neurons in the visual cortex perform in a similar way to neurons in convolutional layers \cite{Poggio_2013}. In the late 1960s, Hubel and Wiesel \cite{Hubel_1962} discovered three types of cells in the visual cortex: simple, complex, and hyper-complex (i.e.\ end-stopped cells). The simple cells are sensitive to the orientation of the excitatory input, whereas the hyper-complex cells are activated by particular types of orientation, motion, and \emph{size} of the stimuli.

The common convolutional layer in a neural network is composed of several fixed-size convolution kernels with trainable/learnable weights (coefficients) \cite{lecun1998a, goodfellow2016}. There are two important properties of a convolutional neuron which differentiates it from a fully connected neuron: 1) it has a local receptive field. 2) it shares its weights with all other neurons at the same layer (assuming a single kernel). Therefore, the same local (non)linear transformation is applied to all regions of the input. Thus, it calculates the same transformation for an input window regardless of its position in the image. However, it is neither scale- nor rotation-invariant, and the size, shape, or orientation of the kernel also affect the output. Though many practitioners often employ basic 3$\times$3 kernels for all tasks, others have tried varied size and shape kernels and different input samplings to improve robustness \cite{li_2017, dai_2017, Ronneberger_15, jeon_2017}. These works are reviewed in Section \ref{relatedadaptive}.

In this study, we describe a new and adaptive model of the convolution layer where the kernel sizes are learned during training. In this unique setting, a single convolution layer can tune and accommodate several kernel sizes at the same time. Such a layer can compute a multi-scale representation from the same input. This is achieved by an additional function which limits and controls the size of the kernel (illustrated in Figure \ref{fig:envelope}). Therefore, the first important question of this paper is: can a differentiable and trainable functional form effectively control the receptive field of a kernel? We tested this ability on an auto-encoder network to learn ordinary image processing operators (e.g., Sobel filter, Gaussian blur). The second question is whether the new adaptively sized convolution kernels can provide any advantage over ordinary fixed-size kernels. In two different network structures, (simple CNN and residual) we substitute the ordinary convolution layers with the adaptive layers to compare their learning and generalization performances. We used the popular MNIST, MNIST-Cluttered, CIFAR-10, Fashion, LFW-Faces (``Labelled Faces in the Wild'') datasets for the comparisons. Finally, we replaced a single convolution layer in a U-net architecture with an adaptive layer and tested it in segmentation.

The main contributions of the current paper are as follows: 
1) a formal description of the two-dimensional adaptive kernel model based on a Gaussian envelope function, 
2) a demonstration that the adaptive envelope makes the kernels less prone to overfitting than ordinary large kernels,
3) a demonstration that their performance is comparable to or better than the ordinary 3$\times$3 kernels which are commonly used in vision applications.

\begin{figure}
	\centering   
	\subfloat[Grid]{\label{fig1:a}\includegraphics[width=0.25\linewidth]{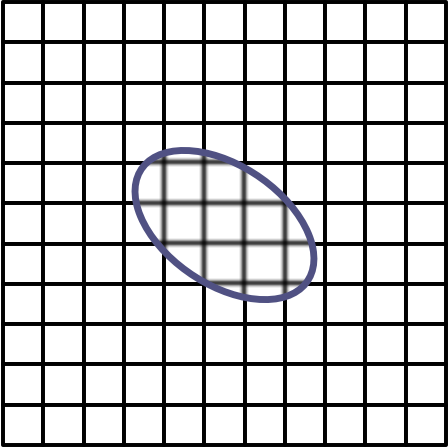}}
	\subfloat[$U$]{\label{fig1:b}\includegraphics[width=0.25\linewidth]{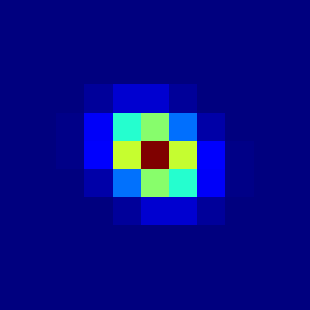}}
	\subfloat[$W$]{\label{fig1:c}\includegraphics[width=0.25\linewidth]{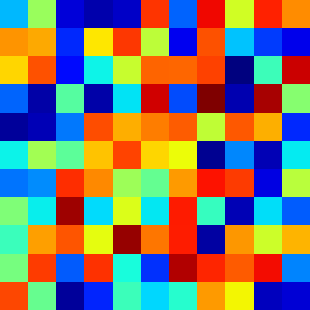}}
	\subfloat[$U\circ W$]{\label{fig1:d}\includegraphics[width=0.25\linewidth]{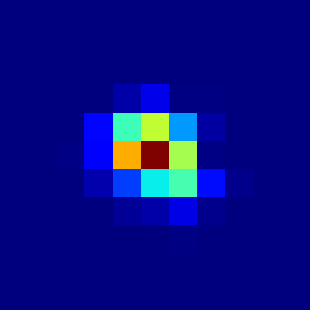}}	
	\caption{Illustration of the proposed weight envelope: (a) An arbitrary differentiable envelope function on a base grid. (b) An example Gaussian envelope ($U$). (c) Randomly generated weights ($W$). (d) Weights are masked by envelope (b) through element-wise multiplication ($U\circ W$).}
	\label{fig:envelope}
\end{figure}

\section{Related Works}\label{relatedadaptive}

The concept of receptive field has attracted attention since the earliest studies of artificial neural networks. The stacked topology of networks enables a neuron in the deeper layers of the network to have an enlarged effective field of view on the input. Recently Luo \cite{luo2016}, found that the effective receptive field grows with the square root of the depth, and in contrary to general belief, the receptive fields of the top-layer neurons may not extend to cover the whole input domain.

The importance of the receptive field in convolution was spotted by signal processing researchers, long before the deep learning community focused on it. The most relevant work is on \textit{atrous} (dilated) convolution \cite{Holschneider_89} which used up-sampled convolution kernels by inserting zeros between the coefficients. The atrous or dilated convolution is widely used in a range of deep learning applications where multi-scale processing is crucial. The applications include image classification \cite{yu_2017}, semantic segmentation \cite{Chen_2017}, speech synthesis\&recognition \cite{Oord_16}, and image denoising \cite{tian2020}.

Dilated convolution enlarges the effective receptive fields. It has been used to provide multi-scale representation in various network configurations \cite{yu_2016, yu_2017, guo_2020, li_2019}. However, it does not solve the problem of scale completely by itself. The network architectures such as U-net aim to increase the scale tolerance of the network by creating multi-scale feature maps \cite{Ronneberger_15}, whereas others such as inception included parallel convolution paths containing different fixed-size convolutions to extract multi-scale information \cite{szegedy2016}.

The shape or orientation of the kernels was also a concern. For example, Li et al.\ \cite{li_2017} studied optimizing kernel shapes using Lasso to create arbitrary shape kernels for audio inputs, as an alternative to commonly used square kernels optimal for natural images. Weiler et al.\ \cite{Weiler_2017} employed steerable kernels trained from a harmonic functional basis to create orientation-sensitive kernels.

The receptive field of a convolution operation can be changed by varying the locations where input is sampled while maintaining the size of the kernel. The ordinary convolution uses fixed sampling locations with respect to the current position ($i$) of the kernel, (e.g., $\lbrace i-2$, $i-1$, $i$, $i+1$, $i+2 \rbrace$), whereas dilated convolution would sample sparsely (e.g., $\lbrace i-4$, $i-2$, $i$, $i+2$, $i+4\rbrace$) using a fixed sampling parameter. The active convolution model \cite{jeon_2017} attempted to learn input sampling offsets ($pn$) from training data, (e.g., $\lbrace i+p0$, $i+p
1$, $i+p2$, $i+p3$, $i+p4\rbrace$). Similarly, Dai's deformable convolution model \cite{dai_2017} dynamically computed sampling offsets per input and per location by performing additional convolutions on the input.

The adaptive model proposed here differs from these approaches in three ways: 1) it does not change the way that input is sampled, 2) it does not use secondary convolutions to compute parameters, 3) it learns the kernel size from training data. The parameters are static and not computed per input, meaning that after training the kernels are fixed.

The proposed model can be seen as an aperture-only, case-specific exemplar of a generalized form, the adaptive locally connected neuron \cite{tek_2019, tek_2018} which can learn its receptive field location and aperture using a Gaussian focus attachment.

\section{Method}

The proposed kernel model learns the receptive field size of the kernel by training a smooth envelope function that can grow or shrink in a base kernel grid. The following sections explain the role of the envelope function, provide an appropriate functional form to construct the envelope, and discuss its parameters. Before starting, note that although commonly referred to as convolution, the operation that is studied and used in neural networks is more appropriately termed cross-correlation. Therefore, for mathematical consistency, we continue with the term cross-correlation instead of convolution, although we use the terms interchangeably for the sake of consistency with the literature. 

A 2-D-matrix cross-correlation computes its output $O=[o_{i,j}]$ by calculating the weighted sum of the ($n \times n$ shaped) kernel coefficients $W=[w_{k,l}]$ times the input $X=[x_{i,j}]$ across all possible locations $i,j$. Therefore, the output matrix of valid size $(M-n)\times (N-n)$ can be expressed in the following form:
\begin{multline}\label{corr}
O = X \star W = \\ \left[ \sum_{k=-\lc n/2 \rc}^{\lc n/2 \rc} \: \sum_{l=-\lc n/2 \rc}^{\lc n/2 \rc}
x_{(i+k , j+l)} w_{(k+\lc n/2 \rc, l+\lc n/2 \rc)}\right]_{i=\lc n/2 \rc,j=\lc n/2 \rc}^{M-\lc n/2 \rc,N-\lc n/2 \rc}
\end{multline}
where, for simplicity, we can ignore the precise offsets, subscripts (e.g., $w_{k+\lc n/2 \rc, l+\lc n/2 \rc}$) and index limits to use the following form (\ref{corr2}) which is sufficient for our discussions:
\begin{equation}\label{corr2}
O = X \star W = \left[ \sum_{k}^{n} \: \sum_{l}^{n}
x_{(i+k , j+l)} \: w_{k, l}\right]_{i,j}^{M,N}
\end{equation}

\subsection{The envelope function}
\label{sec:methods:envelope}

In the adaptive model, the kernel coefficient matrix $W$ is paired with an envelope $U=[u_{k,l}]$ which controls kernel growth through an element-wise multiplication (i.e.\ the Hadamard product):
\begin{equation}
O = X \star (W \circ U) = \left[ \sum_{k}^{n} \: \sum_{l}^{n}
x_{(i+k , j+l)} \: w_{k, l} \: u_{k, l}\right]_{i,j}^{M,N}
\label{UWconv}
\end{equation} 

It may seem as if we are adding just another weight; however, the envelope coefficients are not independent of each other. Here, we define the envelope on a two-dimensional Euclidean space since it is the most common case which can be generalized to further dimensions. We omitted input channels in our notation; if the input contains channels, the weight matrix is three or more dimensional, so the envelope $U$ must be repeated on that dimension. As illustrated in Figure~\ref{fig:envelope}, the envelope resides in a base grid which is also the kernel domain. Let us assume an $n \times n$ base grid for an odd-sized square kernel; and let $U_f$ be a smooth and differentiable function defined in this domain by a parameter set $\bm{\theta} \in \mathbb{R}^p$ (\ref{eq:grid}):
\begin{equation}
\label{eq:grid}
\begin{gathered}
U_f \colon ((k,l), \bm{\theta})\mapsto u_{k,l} \in \mathbb{R} \qquad  \text{where} \\ \quad \{(k, l) \: | \: k,l \in \{1, 2, .. , n\}\} \quad \text{and} \quad
\bm{\theta} = \{\theta_1, \theta_2, .. \theta_p\}
\end{gathered}
\end{equation}

Thus, a functional form can be chosen or designed for $U_f$ to control the envelope shape represented by the coefficients $u_{k,l}$ which mask the weights. When $U_f$ is differentiable with respect to the parameters $\bm{\theta}$, the error derivatives can be calculated using the chain rule; and the updates can be performed using (\ref{eq:genericderive}):
\begin{equation}\label{eq:genericderive}
w_{k,l}^{'} \mathrel{{:}{=}} w_{k,l} - \eta \frac{\partial E}{\partial O} \frac{\partial O}{\partial w_{k,l}}\qquad
\theta_{p}^{'} \mathrel{{:}{=}} \theta_{p} - \eta \frac{\partial E}{\partial O} \frac{\partial O}{\partial \theta_{p}}\qquad
\end{equation}
where $\eta$ denotes the learning rate, $w'_{k,l}$ and $\theta_{p}^{'}$ denote the updated kernel weight coefficient and envelope parameter respectively, and $E$ denotes an error term.
Though they seem disconnected, the updates of the envelope coefficients and weights are related. We elaborate this point by inspecting the partial derivatives of $E$ with respect to $w_{k,l}$ and $\theta_p$. The expression for the derivative ${\partial E}/\partial w_{k,l}$ (\ref{partialw}) includes the focus coefficient $u_{k,l}$ as a scaler coefficient:
\begin{align}
\frac{\partial E}{\partial w_{k,l}} & = \sum_{i}^{M}
\sum_{j}^{N} \frac{\partial E}{\partial {o}_{i,j}} \:
\frac{\partial {o}_{i,j}}{\partial w_{k,l}} 
= \sum_{i}^{M} \sum_{j}^{N} \frac{\partial E}{\partial {o}_{i,j}} \: {x}_{(i+k,j+l)} \: u_{k,l}\\
& = u_{k,l} \sum_{i}^{M} \sum_{j}^{N} \frac{\partial E}{\partial {o}_{i,j}} \: {x}_{(i+k,j+l)}
\label{partialw}
\end{align}
Thus, the envelope not only controls the forward signal but also affects the weight updates. Likewise, we calculate the derivative with respect to the envelope parameter $\theta_p$:
\begin{equation}
\frac{\partial E}{\partial \theta_p}  = \sum_{i}^{M}
\sum_{j}^{N} \frac{\partial E}{\partial {o}_{i,j}} \:
\frac{\partial {o}_{i,j}}{\partial \theta_p} \end{equation} 
where
\begin{equation}
\frac{\partial {o}_{i,j}}{\partial \theta_p} = \sum_{k}^{n} \sum_{l}^{n}  {x}_{i+k,j+l} \: w_{k,l} \frac{\partial u_{k,l}}{\partial {\theta}_{p}}.
\end{equation}
Thus, we can write the following expression:
\begin{equation}
\frac{\partial E}{\partial \theta_p} = \sum_{i}^{M} \sum_{j}^{N} \frac{\partial E}{\partial {o}_{i,j}} \left( \sum_k^{n} \sum_l^n  {x}_{i+k,j+l} \: w_{k,l} \frac{\partial u_{k,l}}{\partial {\theta}_{p}}\right)
\label{partialtheta}
\end{equation}
We see that the derivative with respect to the envelope parameter is accumulated over both the input image and kernel, unlike the weight derivative (\ref{partialw}) which is only accumulated over the whole image. This is because $u_{k,l}$ values are not independent of each other.

\subsection{Choosing an Envelope Function}
\label{sec:methods:gaussian}
A Gaussian form is the primary candidate for the envelope function because it is continuous and differentiable, and it neither creates nor enhances extrema \citep{Lindeberg2011}. Its center parameter ($\bm{\mu}$) controls the position, the covariance parameter ($\Sigma$) smoothly controls the orientation and spread of the form, and $s$ performs the normalization:
\begin{equation}
\label{eq:d_theta}
U_f(\bm{g}, (\mathbf{\mu}, \mathbf{\Sigma})) = s\; e ^{ - \frac{1}{2} (\mathbf{g}-\mathbf{\mu})' \Sigma^{-1} (\mathbf{g}-\mathbf{\mu})}
\end{equation}
In two-dimensional Euclidean space $\bm{g}\in \mathbb{R}^2$, the center is two-dimensional, $\bm{\mu}=\langle\mu_x,\mu_y\rangle$ and covariance is a 2$\times$2 matrix $\Sigma =  \bigl[ \begin{smallmatrix} \sigma_{xx}^2 & \sigma_{xy}^2 \\ \sigma_{yx}^2 & \sigma_{yy}^2 \end{smallmatrix} \bigr] $. However, we exclude the rotation (and the ellipsoid kernels) from the current discussion (for an extended introduction see \cite{ilker_tez, cam_2018}). Although it is possible to train the kernel position ($\bm{\mu}$) as in \cite{jeon_2017}, we did not observe any benefit from doing so in our preliminary studies. Hence, here $\mu$ is initialized to the center of the grid and not trained. There remains only one trainable parameter $\sigma_u$ which controls the size of the circular envelope shape in the set of parameters: $\bm{\theta} = \{{\mu_x=0.5n,\;\mu_y=0.5n,\; \sigma_u}\}$.

\begin{figure}[t]
	\centering
	\subfloat[Weight, smaller envelope, effective kernel on a 9$\times$9 grid]{\label{fig_env:a}\includegraphics[width=0.95\linewidth]{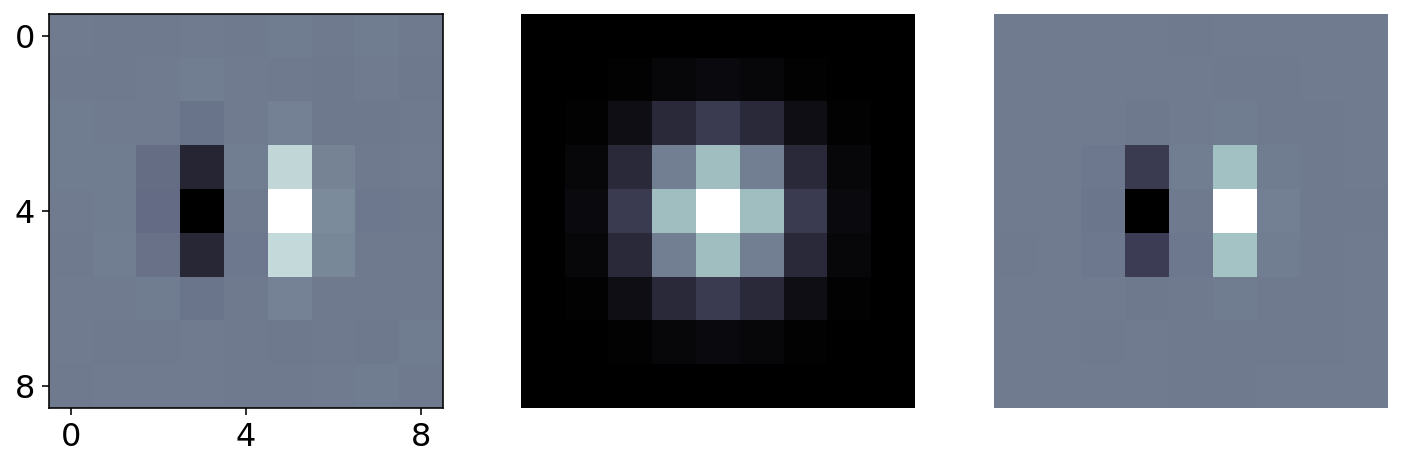}}\\
	\subfloat[Weight, larger envelope, effective kernel on a 9$\times$9 grid]{\label{fig_env:b}\includegraphics[width=0.95\linewidth]{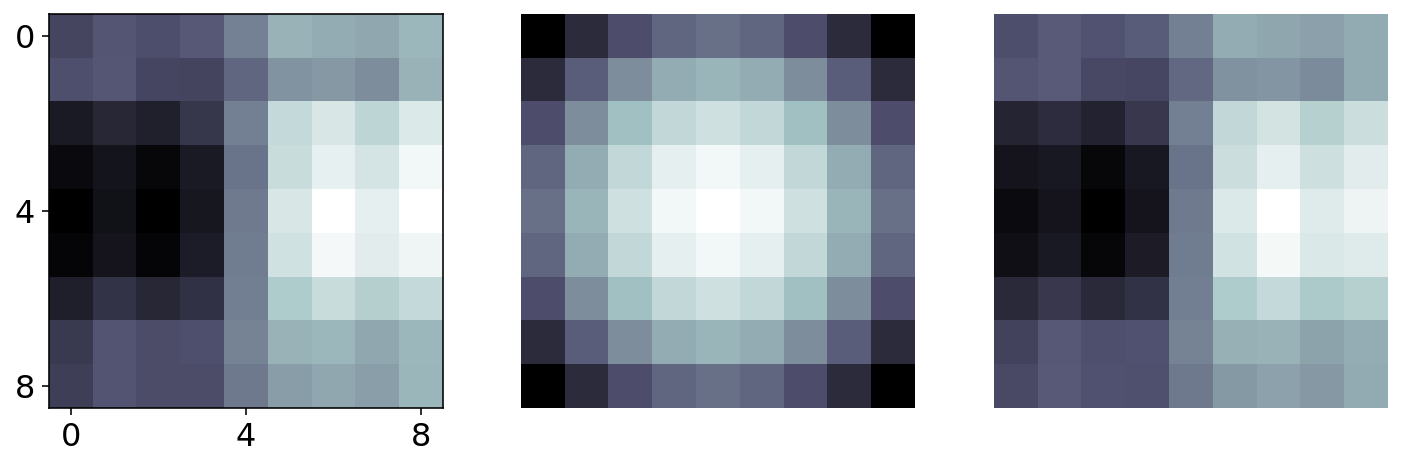}}	
	\caption{Examples of envelopes and effective (product) kernels.}
	\label{fig_env}
	
\end{figure}

During the feed-forward execution the envelope function $U_f$ is computed on the normalized grid coordinates $\bm{g}=\langle k/n,l/n\rangle$ with the current aperture $\sigma_u$ to produce envelope coefficients $u_{k,l}$ which are multiplied element-wise with the weights $w_{k,l}$ prior to the convolution. Figure \ref{fig_env} depicts the weight kernels, envelope matrices and product kernels ($W\circ U$) for two example cases with relatively smaller and larger aperture ($\sigma_u$) values.

Let us denote $l_2$-norm of a vector as $\|\bm{x}\|_2=\sqrt{x_1^2+x_2^2}$, for any given $\bm{x}=\langle x_1,x_2\rangle$. Then the partial derivative with respect to $\sigma_u$, which can replace ${\partial u_{k,l}}/{\partial {\theta}_{p}}$ in (\ref{eq:d_theta}), can be expressed as below (\ref{partial}):
\begin{equation}
\frac{\partial U_f}{\partial \sigma_u} = s \frac{ \|\bm{g}-\bm{\mu}\|_2^2}{2\sigma_u^3} e^{(-\|\bm{g}-\bm{\mu}\|_2^2/(2\sigma_u^2))}
\label{partial}
\end{equation}
\subsection{Initialization of Envelope Parameters}
Recent studies demonstrated that the initialization of weights in a neural network is crucial to improve its training and generalization capacity \citep{he2015, glorot2010}. These studies usually inspect the forward signal and backward gradient flows to suggest an optimal weight-initialization strategy. A common approach is to adjust the variance of the weights so that layer inputs and outputs have equal variance. However, in the adaptive kernel, the envelope coefficients scale the weights and change the variance of the propagated signals. Moreover, since the total fan-in of an adaptive kernel is larger than its effective fan-in it is not clear what value should be used for calculating the weight variance as recommended by common initialization schemes \citep{he2015, glorot2010}.

Nevertheless, we could derive an appropriate initialization variance for the weights in the envelope's presence. However, during the training, the updates to $\sigma_u$ would change the envelope, the product kernel, and the output variance. Therefore, we approach this problem from an alternative perspective where we normalize and scale the envelope $U$ to keep the variance of the weights unchanged by the element-wise multiplication operation $W \circ U$. Although it is not possible to keep the variance of the individual weights $w_{k,l}$ unchanged, it is possible to maintain the mean of the variances. Let us write single summation $\sum_{k,l}^{nxn}$ instead of $\sum_{k}^{n}\sum_{l}^{n}$ to simplify the notation and define the mean of variances along an $n\times n$ matrix $A=[a_{k,l}]_{n\times n}$ by:
\begin{equation}
\MVar [A]=\frac{1}{n^2} \sum\limits_{k,l}^{\nn}\Var(a_{k,l})
\end{equation}

Then, Theorem \ref{the1} states that the mean of the variances of the weights will be unchanged by the Hadamard multiplication of the envelope matrix if the mean of the expected value of the squared envelope coefficients is 1.
\begin{theorem}
	\begin{flushleft}
		$ \displaystyle\MVar[W\circ U]=\MVar[W]=\sigma_{w}^2\quad$ when
	\end{flushleft}
	$\displaystyle	 \frac{1}{n^2} \sum_{k,l}^{\nn}\mathbb{E}(u_{k,l}^2)=1$.
	\label{the1}
\end{theorem}

\begin{proof}
	Since $u_{k,l}$ and $w_{k,l}$ are independent for any $k,l$ and weights are i.i.d with zero mean (i.e.\ $\mathbb{E}(w_{k,l}=0$)), we have
	\begin{align}
	\Var(u_{k,l}  w_{k,l}) &= 
	\mathbb{E}(u_{k,l}^2 w_{k,l}^2)-\left[\mathbb{E}(u_{k,l})\mathbb{E}(w_{k,l})\right]^2 \\
	&=\mathbb{E}(u_{k,l}^2)\mathbb{E}(w_{k,l}^2)=\mathbb{E}(u_{k,l}^2)\sigma_{w}^2
	\end{align}
	Next, we consider the mean of variances.
	\begin{align}
	\MVar(U \circ W) &=
	\frac{1}{n^2}\sum_{k,l}^{\nn} \Var(u_{k,l} w_{k,l}) \\ &=\sigma_{w}^2\frac{1}{n^2}\sum_{k,l}^{\nn} \mathbb{E}(u_{k,l}^2 )\\
	&=\MVar(W)\frac{1}{n^2}\sum_{k,l}^{\nn} \mathbb{E}(u_{k,l}^2).
	\end{align}
	Hence the result follows.
\end{proof}

Likewise, we may consider the backward propagation of the error variance. Using (\ref{partialw}) and the assumptions $\mathbb{E}({x}_{(i+k,j+l)})=0$ and $\mathbb{E}(w_{k_l})=0$, we formulate the mean of the gradient variances of the weights as follows (\ref{the2})(see the appendix for the derivation):
\begin{align}
\MVar \bigg(\frac{\partial E}{\partial W} \bigg) & =  \frac{1}{n^2} \sum_{k,l}^{\nn}\Var \bigg({\frac{\partial E}{\partial w_{k_l}}}\bigg)\\ &=\sigma_{x}^2 \displaystyle  \bigg[ \frac{1}{n^2} \sum_{k,l}^{\nn}\mathbb{E}(u_{k,l}^2) \bigg] \sum_{i,j}^{M,N} \mathbb{E} \bigg[ \big(\frac{\partial E}{\partial {o}_{i,j}} \big)^2 \bigg]
\label{the2}
\end{align}

Thus, Eq.(\ref{the2}) states that the envelope coefficients affect the mean of the gradient variances of the weights minimally if the mean expected value of the squared envelope coefficients is 1.0. However, the envelope function is a deterministic function of the random parameter $\sigma_u$ which has an unknown probability density, because it will be learned by the network. However, we see that the function $U_f$ can be scaled so that $u_{k,l}$ sum to a constant value irrespective of the $\sigma_u$ value. Therefore, satisfying the condition on Theorem \ref{the1} translates to a condition on the norm of the $\|U\|_2^2= \sum_{k,l}^{n\times n} u_{k,l}^2=n^2$. In practice, a convolution layer would have more than one kernel to calculate multiple outputs. Hence, in the forward run, each kernel $q$ calculates its envelope using its own $\sigma^q_u$, then normalizes itself using $s_u^q$:
\begin{equation}
s_u^q= \frac{n}{\sqrt{\sum\limits_{\mathbf{g} \in A \times A} \left(e ^{ - (\frac{\mathbf{g}-\mathbf{\mu})^T(\mathbf{g}-\mathbf{\mu}))}{ 2(\sigma^{q}_u)^2}}\right)^2}}
\label{norm}
\end{equation}

To test this proposition empirically, we set up a simple experiment. In a loop of increasing aperture ($\sigma_u$) values, we calculated the corresponding envelope matrix and also randomly sampled weight matrices of size $n \times n\times channels \times filters$ from a uniform distribution (normal distribution was also tested). The weight sample variance was calculated along the channel and filter dimensions. Figure \ref{fig:wuvar} shows the mean of the variances of the weight matrices (MVar[W]) against the mean of the variances of the product kernels (MVar[$W\circ U$]) for increasing envelope (U) aperture $\sigma_u$ value. It can be seen that when we normalized $U_f$ using (\ref{norm}), negligibly small deviations occurred in the product kernel variance at very low aperture values, whereas the larger aperture envelopes maintained the mean of the weight variances perfectly. 

\begin{figure}
	\centering
	\includegraphics[width=0.9\linewidth]{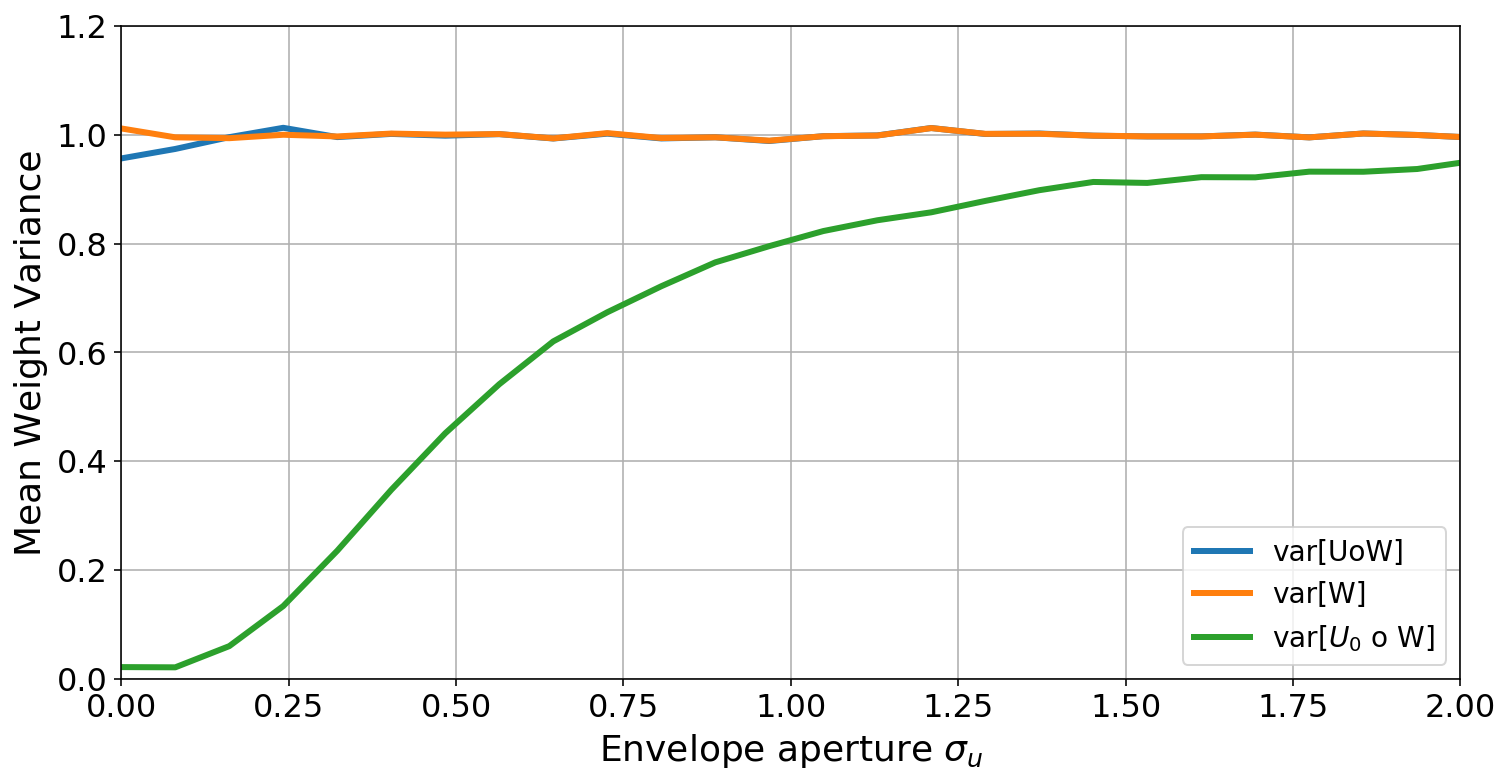}
	\caption{Mean of the kernel variances against increasing aperture $\sigma_u$. $U$: normalized, $U_0$: unnormalized envelope.}
	\label{fig:wuvar}
\end{figure}

On the other hand, we do not have formal guidance on the initialization of $\sigma_u$, except that it must be positive and non-zero. In practice, we have noticed that initializing $\sigma_u$ in the range $[1/\text{n}, \text{n}]$ works well ($n$=kernel size). However, during training the $\sigma_u$ value must be monitored and clipped to stay above a value (e.g., $1/\text{n}$) to prevent over-shrinking by high gradient values and fluctuations that may occur.

\section{Experiments}\label{sec:exp}

We divided the experiments into four sections. First, we investigated whether the proposed adaptive kernel can learn common image processing filters. Second, we set up a simple convolutional network and compared the adaptive kernels against the ordinary convolutional kernels in their learning and generalization performance. Next, we repeated the same comparison in a popular deep architecture, ResNet \cite{resnet}. Finally, we tested the adaptive kernels in a U-net \cite{Ronneberger_15} architecture for segmentation. We implemented the proposed model in Python 3 using Keras \& Tensorflow \cite{keras}. All code and a demo are available in \cite{gitcodeacnn}.

\subsection{Learning Basic Image Processing Kernels}\label{sec:exp_filt}
We set up a simple auto-encoder network to test whether the new adaptive kernels are able to synthesize some basic image processing kernels. The network configuration can be found in the supplementary materials. The network took a single image (e.g., Figure \ref{fig:auto_enc:a}) as input to learn the outputs of nine different image processing kernels of size 9$\times$9 pixels. The targets included the output of 3$\times$3 Laplace, horizontal and vertical Sobel kernels, Gauss smoothing kernels of different variance, and applications of Laplace and Sobel to the Gauss-smoothed outputs, as shown in Figure \ref{fig:auto_enc:e}. Figure \ref{fig:auto_enc:f} shows the outputs of the network after 500 training iterations using stochastic gradient descent optimizer with a learning rate and momentum of 0.1 and 0.9, respectively. The training converged after 150-200 iterations, as seen in Figure \ref{fig:auto_enc:c}. In addition, we observed that the minimum square error of the adaptive convolution network was slightly lower than an ordinary convolution network which contain kernels of equal size (9$\times$9). However, both networks were able to learn the kernels. Figure \ref{fig:auto_enc:d} shows that most of the aperture parameters ($\sigma_u$) converged at around 150-200 epochs in the adaptive network. The initial weight ($W$) and envelope ($U$) kernels are shown in Figures \ref{fig:auto_enc:g} and \ref{fig:auto_enc:h} with the learned envelope and final product kernels ($U\circ W$) are given in Figures \ref{fig:auto_enc:i} and \ref{fig:auto_enc:j}, respectively. It can be seen that the envelopes were successfully learned in the presence of the weights and vice versa. In conclusion, the adaptive kernels were able to learn basic image processing kernels of different size and character.

\begin{figure*}
	\centering     
	\subfloat[Input]{\label{fig:auto_enc:a}\includegraphics[width=0.2\linewidth]{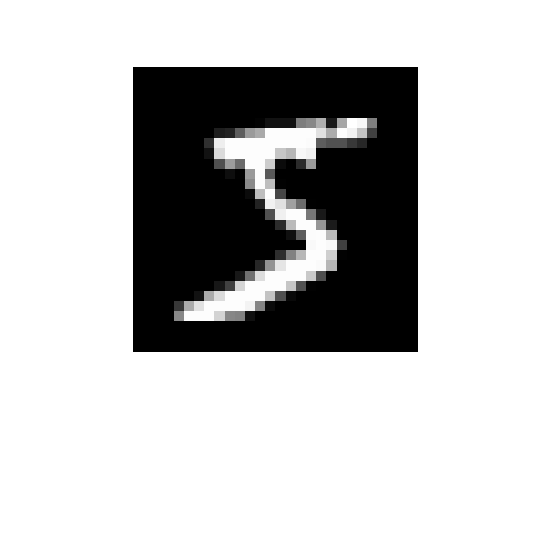}}
	\subfloat[Train Loss]{\label{fig:auto_enc:c}\includegraphics[width=0.35\linewidth]{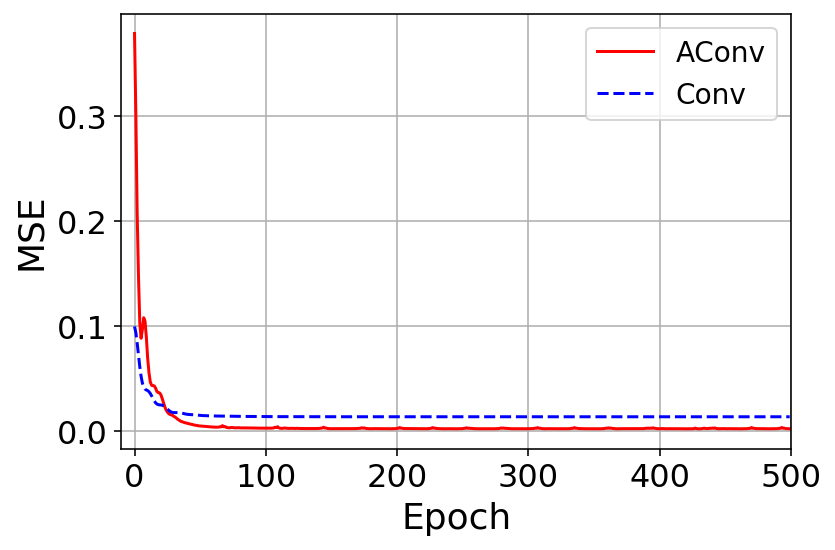}}
	\subfloat[Apertures]{\label{fig:auto_enc:d}\includegraphics[width=0.35\linewidth]{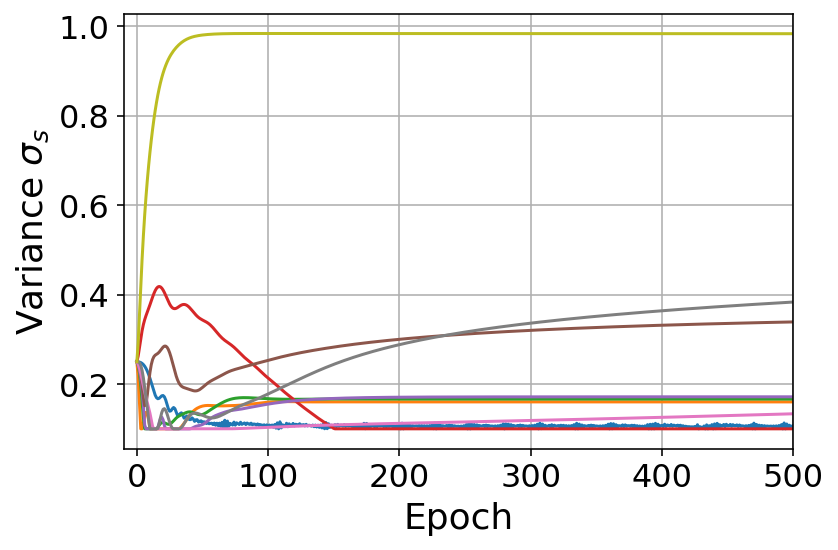}}\\\vspace{-0.0cm}
	\subfloat[Targets]{\label{fig:auto_enc:e}\includegraphics[width=0.98\linewidth]{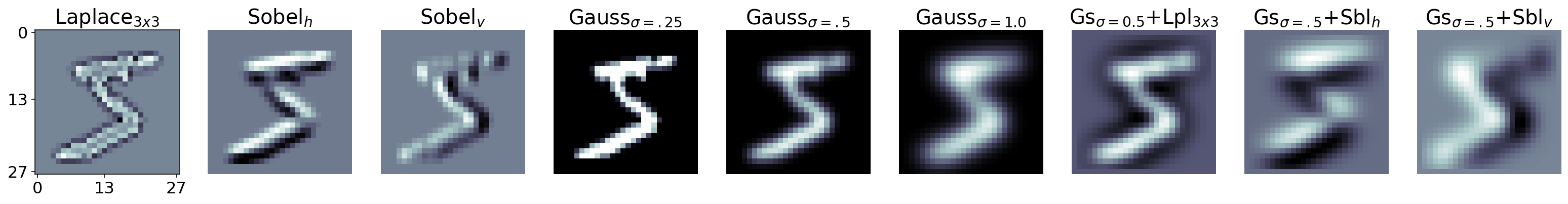}}\\ \vspace{-0.0cm}
	\subfloat[Predictions]{\label{fig:auto_enc:f}\includegraphics[width=0.98\linewidth]{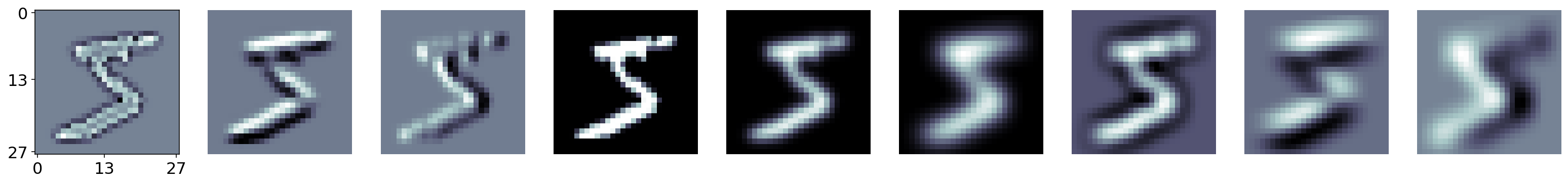}}\\ \vspace{-0.0cm}
	\subfloat[Weights before training]{\label{fig:auto_enc:g}\includegraphics[width=0.98\linewidth]{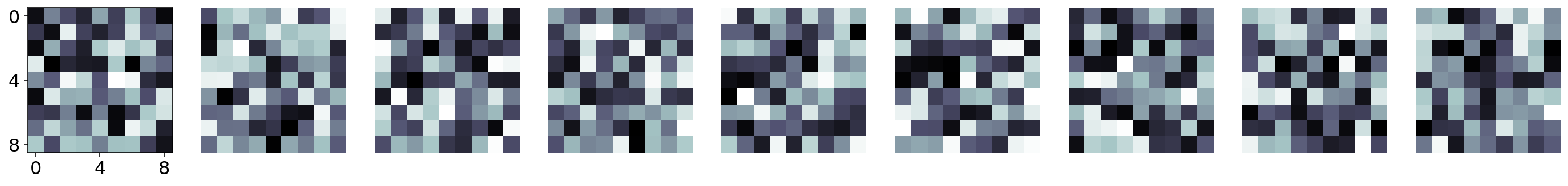}}\\ \vspace{-0.0cm}
	\subfloat[Envelopes before training]{\label{fig:auto_enc:h}\includegraphics[width=0.98\linewidth]{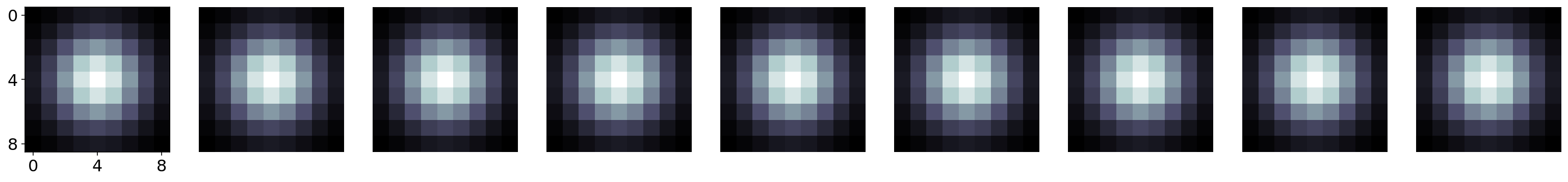}}\\ \vspace{-0.0cm}
	\subfloat[Envelopes after training]{\label{fig:auto_enc:i}\includegraphics[width=0.98\linewidth]{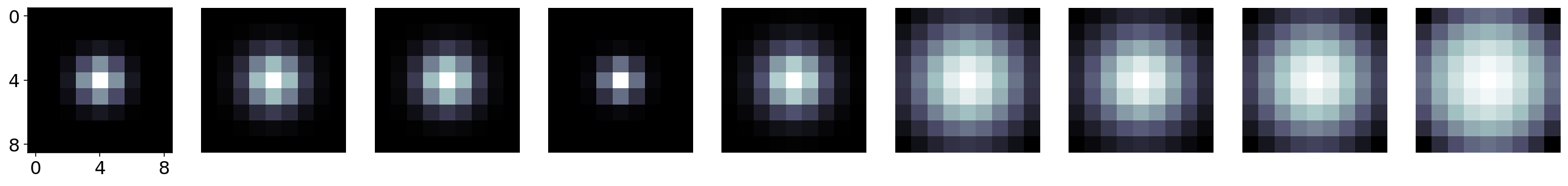}}\\ \vspace{-0.0cm}
	\subfloat[Product kernels after training]{\label{fig:auto_enc:j}\includegraphics[width=0.98\linewidth]{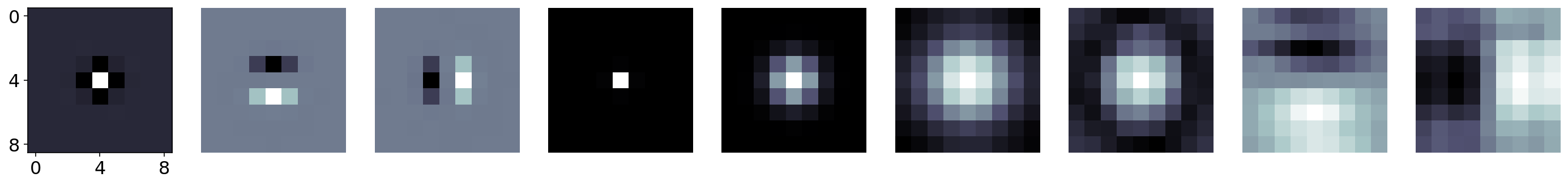}}
	\caption{The auto-encoder with the adaptive convolution kernels learning basic image kernels. a) Input image. b) Mean Squared Loss. c) Change of $\sigma_u$ during training. d) Output images (targets) created by image processing kernels (Laplace, Sobel, Gauss and combinations). e) Network predicted outputs after 2500 training updates. f) Initial weights. g) Initial envelopes. h) Learned envelopes. i) Effective (product) kernels).}
	\label{fig:auto_enc}
\end{figure*}

\subsection{Comparisons in Simple Convolutional Network}\label{sec:exp_simp}
Next, we compared the proposed adaptive kernels (ACONV) to the ordinary convolution (CONV) kernels in a simple convolutional classification network. The basic network configuration is summarized in Table \ref{simp:net} (see the supplementary figures for a plot of the network graph). The network was built using two consecutive convolutional layers (CONV or ACONV) of 32 kernels followed by a single max-pool layer of size 2$\times$2 and a dense classification layer of 256 units surrounded by two drop-out layers. All neuronal layers were followed by batch normalization (BN) and rectified linear activation units (RELU). The network output was formed of $c$ Softmax units, where $c$ was equal to the number of dataset categories.

The popular gray-scale MNIST character recognition dataset \citep{lecun1998a} (MNIST) was the first place to start. More challenging datasets were also used: a cluttered version of MNIST data (CLT), comprised of randomly transformed MNIST samples superimposed on cluttered 60$\times$60 backgrounds \citep{jaderberg2015}; the CIFAR-10 general object classification dataset which is composed of 32$\times$32$\times$3 RGB images of ten concrete categories such as car, plane, bird, horse \citep{cifar10}; and the FASHION (clothes) dataset which is arranged similarly to MNIST \citep{xiao2017} to include 10 categories such as t-shirt, pullover, and coat. These almost-standard datasets had already been separated into training (60000) and test (validation) instances (10000). The tests also included the ``Faces in the Wild'' dataset (LFW-Faces) \citep{lfw} as a benchmark for face verification. The LFW-Faces set contains 13233 images of 5749 people; in order to reduce the number of output classes, individuals with less than 20 images were excluded from the experiments, resulting in a dataset of 3023 (2267 training, 756 validation) images of 62 people.

All comparisons were repeated with different kernel sizes $n$ (3$\times$3, 5$\times$5, 7$\times$7, 9$\times$9) and with 5 different random initializations. To ensure a fair comparison, we fine-tuned the ordinary kernel network to get the best validation accuracy, then replaced the ordinary kernel with the adaptive kernel. We then tuned the model-specific parameters such as the initial $\sigma_u$'s before comparing the two cases. In other words, for the given configuration, we compared the maximum performance we could achieve using the ordinary kernel with the adaptive kernel. We used stochastic gradient descent optimizer with an initial learning rate of $0.1\times\eta_{dset}$ (dataset specific multiplier) and a momentum of 0.9 and gradient clip value of 1.0. We further employed an adaptive learning rate schedule which monitors the validation loss and drops the learning rate by a factor of 0.9 when no improvement is seen in the past 10 epochs. Table \ref{simp:net} lists other important parameters used in training. We initialized the aperture values $\sigma_u^q$ with linearly spaced values in the range [0.1,0.5]. In addition, we attached a clip function to the optimizer to clip the $\sigma_u$ values within range $[1/n, n]$ after each update.

\begin{algorithm}
	\DontPrintSemicolon
	\KwIn{network, dataset, $N_{repeats}$, $N_{epochs}$, $n$: kernel size, $\eta_{dset}$: learning rate multiplier, $m=0.9$: momentum rate. \emph{OPT: SGD (stochastic gradient descent with momentum) or Adam or SGD with Cyclic Schedule}} 
	\KwOut{BestTestResults: list of best test accuracies.}
	\Begin{
		BestTestResults = []\;
		\For{$r \leftarrow 0$ \KwTo $N_{repeats}-1$}{
			EpochAccuracyList = []\;
			trainX,trainy,testX,testy = split(dataset)\;
			\For{$e\leftarrow 0$ \KwTo $N_{epochs}-1$}{
				
				\For{each batch (Xinputs, targets) in (trainX, trainy)}{
					params $\leftarrow $ network.trainableparams\;
					pred $\leftarrow $ network.output(Xinputs)\;
					loss $\leftarrow $ categoricalcrossentropy(targets, pred)\;
					
					updates $\leftarrow $ OPT(loss, params, 0.1* $\eta_{dset}$, $m$, clipvalue=1.0)\;
					
					\If{$type$(network) == focused} {
						updates.append(clip(params.sigma, 1/n, n))\;					
					}			
					network.update(updates)				
				}		
				score $\leftarrow $ accuracyscore(network, testX,testy)\;
				EpochAccuracyList.append(score)
			}
			maxscore $\leftarrow $ max(EpochAccuracyList)\;
			BestTestResults.append(maxscore)\;
			
		}
	}
	\caption{Training, validation, and optimization procedure}
	\label{algo:3}
\end{algorithm} 

Figure \ref{fig:simp:comp} shows a comparison of the mean validation accuracies with respect to training epochs. In all five datasets, the adaptive layers performed better than their fixed-size counterparts. As anticipated, the 3$\times$3 ACONV kernels performed the least effectively, since there is limited room to operate the adaptive aperture. In contrast, the 7$\times$7 and 9$\times$9 kernels often performed the best.

\begin{figure*}
	\centering     
	\subfloat[MNIST]{\label{fig:simp:comp:a}\includegraphics[width=0.5\linewidth]{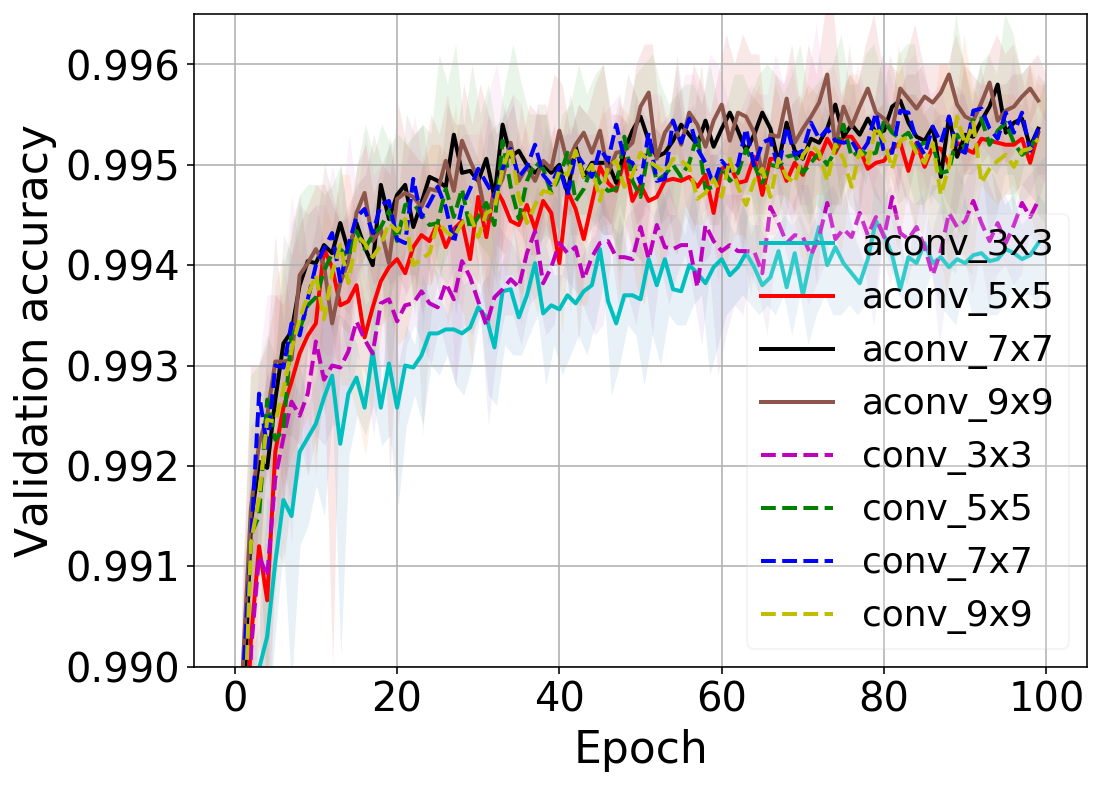}}
	\subfloat[MNIST-CLUT]{\label{fig:simp:comp:b}\includegraphics[width=0.5\linewidth]{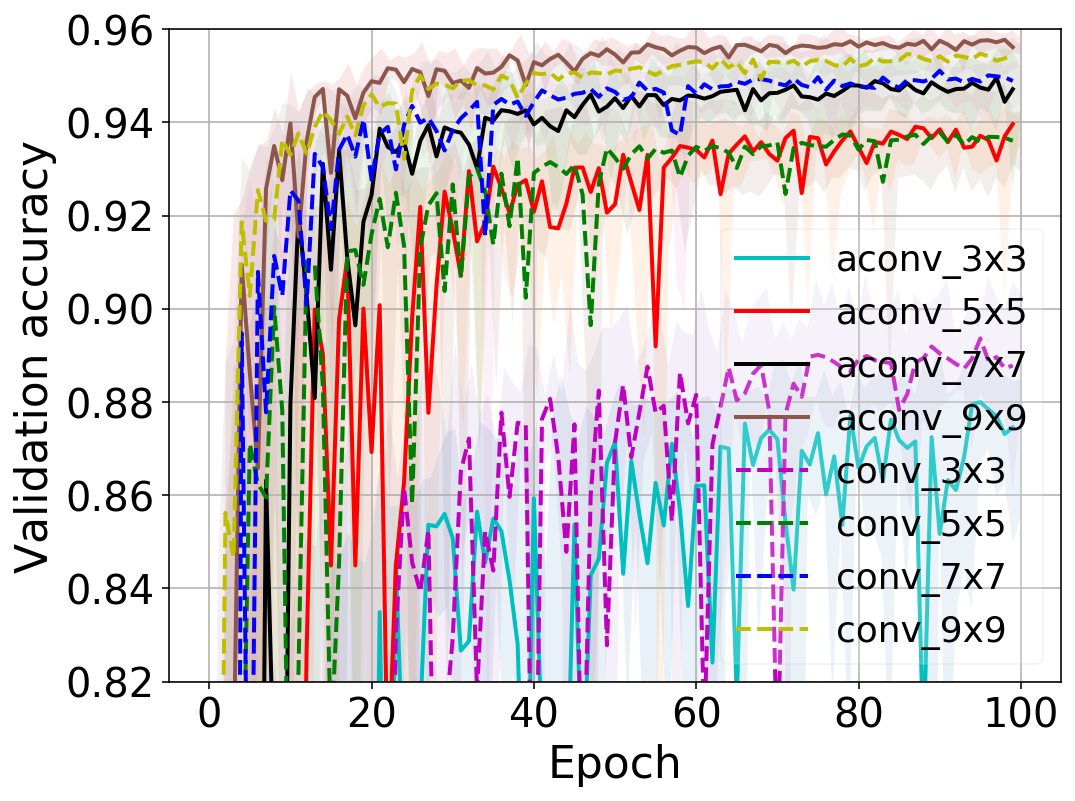}}\\
	\subfloat[CIFAR10]{\label{fig:simp:comp:c}\includegraphics[width=0.5\linewidth]{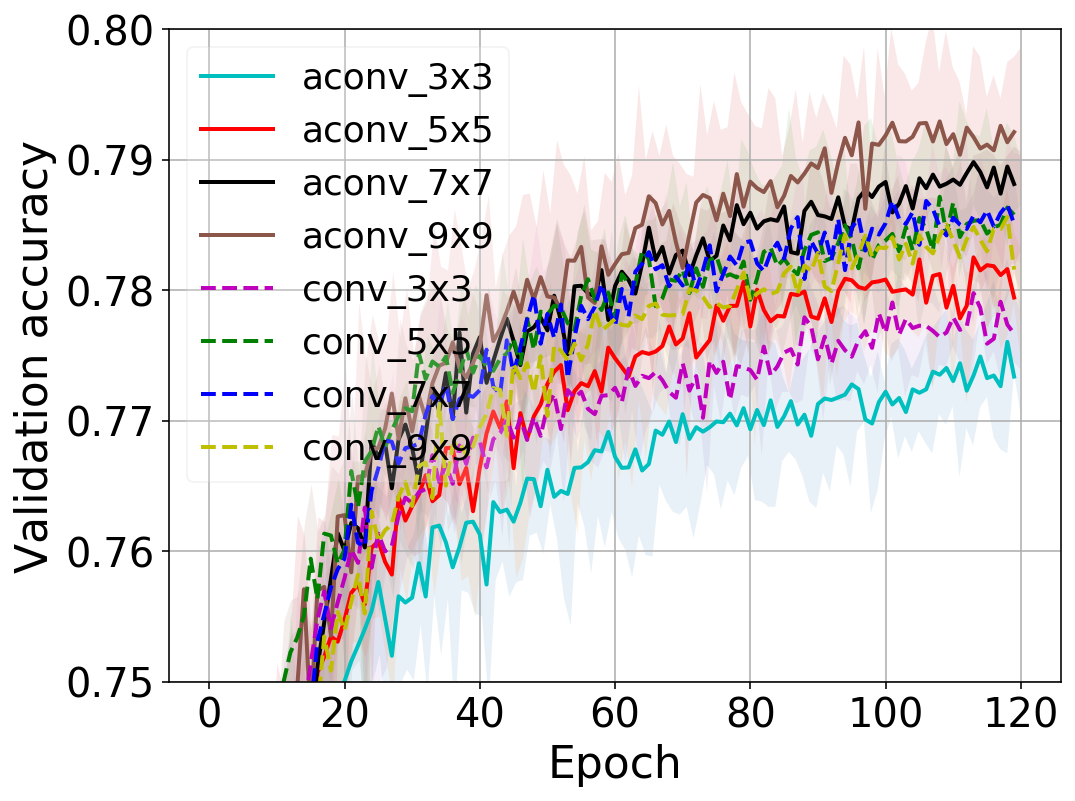}}
	\subfloat[FASHION]{\label{fig:simp:comp:d}\includegraphics[width=0.5\linewidth]{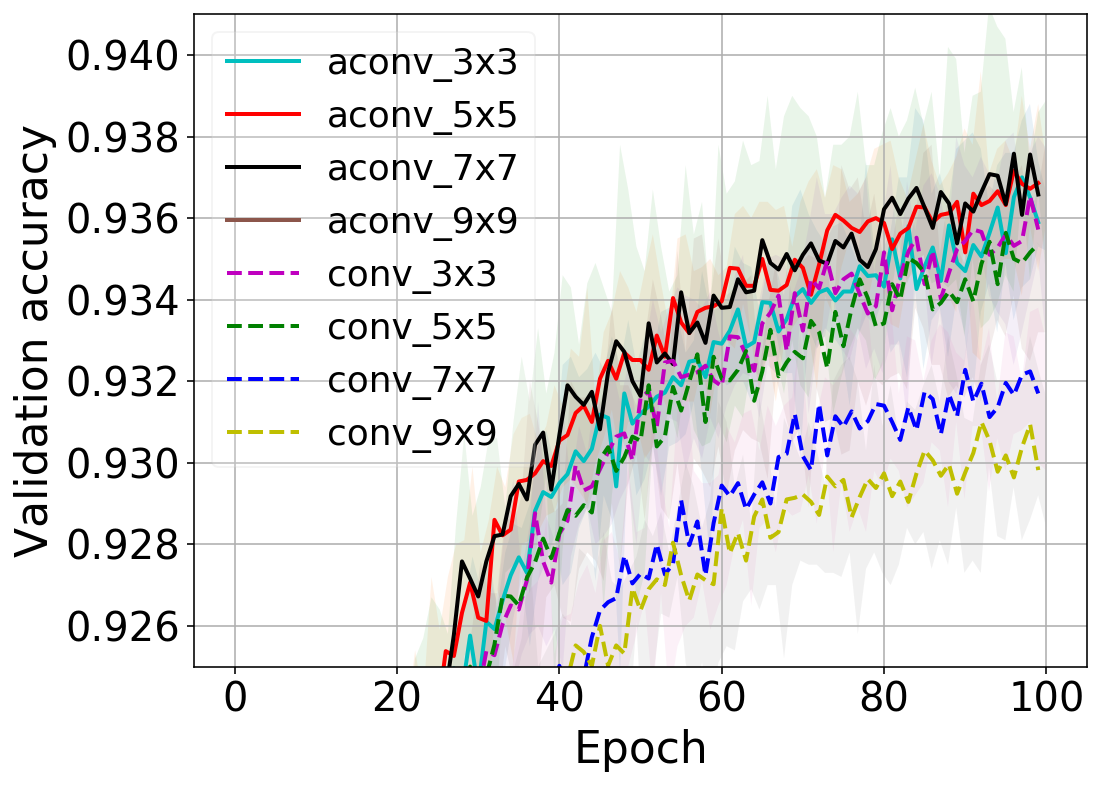}}\\
	\subfloat[FACES]{\label{fig:simp:comp:e}\includegraphics[width=0.5\linewidth]{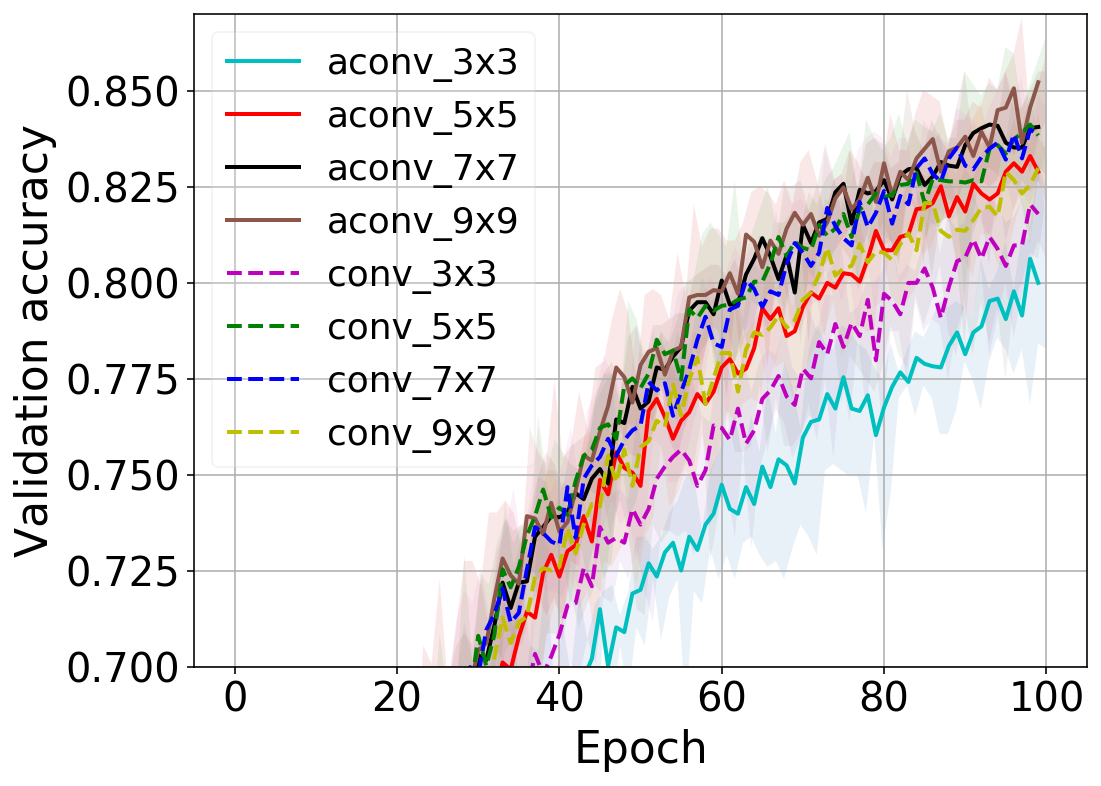}}
	\hfil
	\subfloat[Fashion dataset network (product) weights]{\label{fig:simp:comp:f}\includegraphics[width=0.4\linewidth]{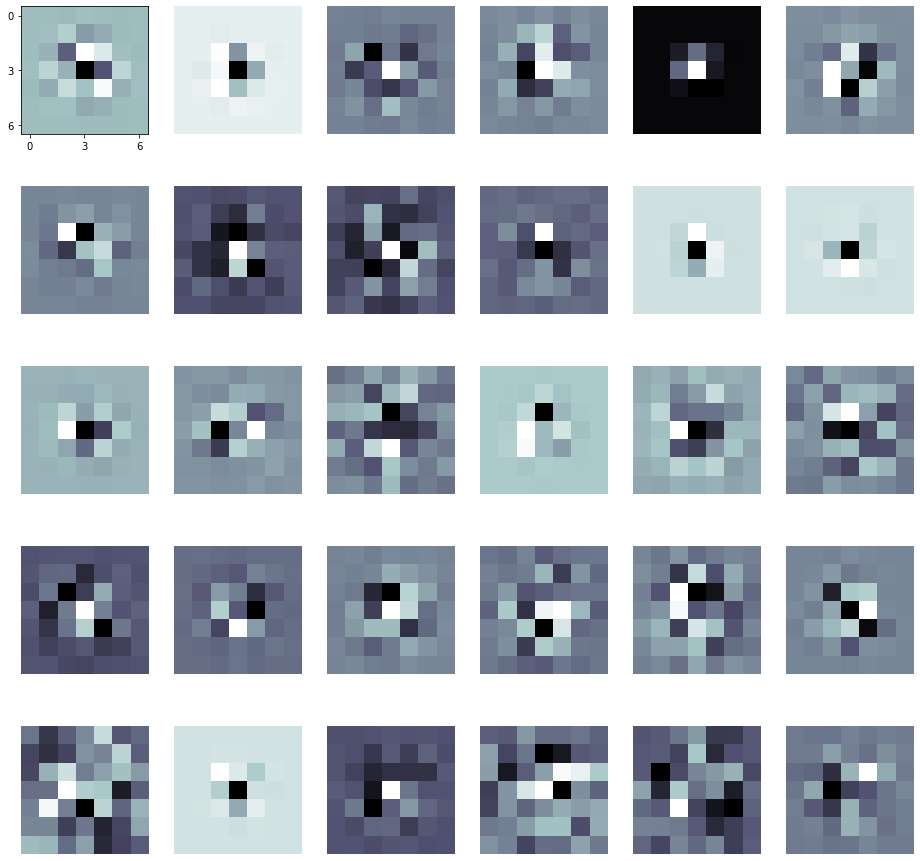}}
	
	\caption{Simple Convolutional Network Comparisons: a-e) Validation Accuracy Plots. f) Learned effective (product) kernels (ACONV-1) on FASHION dataset.)}
	\label{fig:simp:comp}
\end{figure*}

Table \ref{tab:comp:simp} summarizes the results that were calculated using Algorithm \ref{algo:3} across five repeats. Comparing the mean peak validation accuracies, the adaptive filter reached higher validation accuracies in all five datasets. The t-tests compared the means of peak validation accuracies when using the kernel size of the maximum peak performance and results confirmed that the accuracy improvements were all statistically significant. 

Another observation was about the performance of ordinary kernels in different sizes. Although 3$\times$3 sized kernels are preferred in most applications, we observed that the larger kernels produced significantly better results in our setting. Finally, Figure \ref{fig:simp:comp:e} depicts the learned kernels during the Fashion dataset training which demonstrates the varying kernel sizes.

\begin{table*}
	\caption{Left: simple convolutional network. Right: ResNet network and training parameters.}
	\resizebox{\textwidth}{!}{
		\begin{tabular}{  l  r  }
			
			\begin{tabular}[t]{  l | c | c | c | c | c  }
				\toprule
				Simple Network & \small{MNIST} & \small{CLUT} & \small{CFR10} & \small{Fashion} & \small{Faces} \\ \hline
				Num Params & 1.6M & 75K& 1.4M & 0.9M & 0.9M\\ \hline
				Batch       & 128 & 128 & 64 & 256 & 8 \\ \hline
				Augment     & False & False & False& False& True  \\ \hline
				Epoch       & 100 & 100& 120& 100 & 100 \\ \hline
				Dropouts     & 0.5 & 0.5& 0.5& 0.5 & 0.5 \\ \hline
				$\eta_{dset}$ (x 0.1) & 0.1 & 1.0 & 0.1 & 0.1 & 0.01
				\\ \hline
			\end{tabular}
			\hspace{0.5cm}
			\begin{tabular}[t]{  l | c | c | c | c | c  }
				\toprule
				ResNet Params & \small{MNIST} & \small{CLUT} & \small{CFR10} & \small{Fashion} & \small{Faces} \\ \hline
				Filters & 16 & 16 & 16 & 16 & 16  \\ \hline	
				Num Blocks & 1 & 1 & 3 & 2 & 2 \\ \hline
				Num Layers  & 11 & 11 & 20& 14 & 14  \\ \hline
				Num Params & 75K& 75K& 1.4M & 0.9M & 0.9M\\ \hline
				Batch       & 128 & 128 & 128 & 128 & 8 \\ \hline
				Augment     & False & False & True& True& True  \\ \hline
				Epoch       & 50 & 50& 200& 100 & 150 \\ \hline
				Dropout     & 0.5 & -& -& - & - \\ \hline
				$\eta_{dset}$ ($\times$ [1e-3$\rightarrow$0.5]) & 1.0 & 1.0 & 0.1 & 0.1 & 0.1 \\ \hline
				
			\end{tabular}
		\end{tabular}
	}
	\label{simp:net}
\end{table*}

\begin{table*}[t]
	\caption{Validation performances of simple ordinary (CONV) and adaptive (ACONV) convolution networks on popular image classification sets. The two-tailed t-tests are included for each case. N (repeats)=5, p: p-value, \hl{highlights} indicate \hl{*: p-value$<$0.05}. Best size indicates the best performing (and t-test comparison) kernel size.}{
		\resizebox{\textwidth}{!}{%
			\begin{tabular}{@{}lcccccccccc@{}}
				\toprule
				& \multicolumn{2}{c}{MNIST} & \multicolumn{2}{c}{CLT} & \multicolumn{2}{c}{CIFAR-10} & \multicolumn{2}{c}{Fashion} & \multicolumn{2}{c}{LFW-Faces}  \\ \midrule

				& Mn$\pm$std & \multicolumn{1}{c|}{Max} & Mn$\pm$std & \multicolumn{1}{c|}{Max} & Mn$\pm$std & \multicolumn{1}{c|}{Max} & Mn$\pm$std & \multicolumn{1}{c|}{Max} & Mn$\pm$std & \multicolumn{1}{c|}{Max}  \\ \midrule

				CONV & 99.58$\pm$2e-4 & \multicolumn{1}{c|}{99.61} & 95.59$\pm$2e-3 & \multicolumn{1}{c|}{95.9} & 78.74$\pm$2e-3 & \multicolumn{1}{c|}{78.95} & 93.31$\pm$6e-4 & \multicolumn{1}{c|}{93.38} & 83.87$\pm$7e-3 & \multicolumn{1}{c|}{85.66} \\ \midrule
				
				ACONV & \textbf{99.63$\pm$2e-4} & \multicolumn{1}{c|}{{\textbf{99.66}}} & \textbf{95.9$\pm$9e-4} & \multicolumn{1}{c|}{\textbf{96.06}} & \textbf{\textbf{79.63$\pm$5e-3} }& \multicolumn{1}{c|}{\textbf{80.1}} & \textbf{93.84$\pm$2e-3} & \multicolumn{1}{c|}{\textbf{94.12}}& \textbf{85.79$\pm$6e-3} & \multicolumn{1}{c|}{\textbf{86.93}} \\ \midrule

				\small{T-Test (t,p)} & 2.88 & \multicolumn{1}{c|}{\hl{0.02*}} & 2.46 & \multicolumn{1}{c|}{\hl{0.039*}} & 3.28 & \multicolumn{1}{c|}{\hl{0.011*}} & 6.16 & \multicolumn{1}{c|}{\hl{2.7e-4*}} & 3.89 & \multicolumn{1}{c|}{\hl{4.5e-3*}} \\ \midrule
				
				Best size & \multicolumn{2}{c|}{9$\times$9} &  \multicolumn{2}{c|}{9$\times$9} & \multicolumn{2}{c|}{9$\times$9} & \multicolumn{2}{c|}{7$\times$7} & \multicolumn{2}{c|}{9$\times$9}  \\ 
				
				\bottomrule
		\end{tabular}}
	}\label{tab:comp:simp}
\end{table*}

\begin{table*}
	\caption{Validation performances of Residual ordinary (CONV) and adaptive (ACONV) convolution networks (ResNet) in popular image classification sets. The two-tailed t-tests are included for each case. N (repeats)=5, p: p-value, \hl{highlights} indicate \hl{*: p-value$<$0.05}.}{
		\resizebox{\textwidth}{!}{%
			\begin{tabular}{@{}lcccccccccc@{}}
				\toprule
				& \multicolumn{2}{c}{MNIST} & \multicolumn{2}{c}{CLT} & \multicolumn{2}{c}{CIFAR-10} & \multicolumn{2}{c}{Fashion} & \multicolumn{2}{c}{LFW-Faces}  \\ \midrule

				& Mn$\pm$std & \multicolumn{1}{c|}{Max} & Mn$\pm$std & \multicolumn{1}{c|}{Max} & Mn$\pm$std & \multicolumn{1}{c|}{Max} & Mn$\pm$std & \multicolumn{1}{c|}{Max} & Mn$\pm$std & \multicolumn{1}{c|}{Max}  \\ \midrule

				CONV & 99.69$\pm$2e-4 & \multicolumn{1}{c|}{99.71} & 98.93$\pm$4e-4 & \multicolumn{1}{c|}{99.01} & 91.3$\pm$2e-3 & \multicolumn{1}{c|}{91.71} & 93.95$\pm$1e-4 & \multicolumn{1}{c|}{94.12} & \textbf{96.15$\pm$7e-3} & \multicolumn{1}{c|}{\textbf{97.48}} \\ \midrule
				
				ACONV & \textbf{99.70$\pm$1e-4} & \multicolumn{1}{c|}{{\textbf{99.73}}} & \textbf{99.06$\pm$8e-4} & \multicolumn{1}{c|}{\textbf{99.17}} & \textbf{92.21$\pm$3e-3 }& \multicolumn{1}{c|}{\textbf{92.68}} & \textbf{94.72$\pm$2e-3} & \multicolumn{1}{c|}{\textbf{95.01}} & 94.83$\pm$9e-3 & \multicolumn{1}{c|}{96.06} \\ \midrule

				\small{T-Test (t,p)} & 0.67 & \multicolumn{1}{c|}{0.51} & 2.72 & \multicolumn{1}{c|}{\hl{\textbf{0.02*}}} & 5.12 & \multicolumn{1}{c|}{\hl{\textbf{9e-4*}}} & 7.01 & \multicolumn{1}{c|}{\hl{\textbf{1e-4*}}} & -2.19 & \multicolumn{1}{c|}{0.059} \\ \midrule
				
				Best size & \multicolumn{2}{c|}{5$\times$5} &  \multicolumn{2}{c|}{7$\times$7} & \multicolumn{2}{c|}{5$\times$5} & \multicolumn{2}{c|}{7$\times$7} & \multicolumn{2}{c|}{3$\times$3}  \\ 
				
				\bottomrule
		\end{tabular}}
	}\label{tab:comp:res}
\end{table*}

\subsection{Comparisons in Deep Residual Network (ResNet)}\label{sec:exp_res}
Next, we compared the proposed adaptive kernel (ACONV) to ordinary convolution (CONV) kernels in a modern successful network architecture setup for classification, ResNet \cite{resnet}. For the comparison, we redefined the basic convolutional block to employ either an adaptive or ordinary convolutional layer selectively (see the supplementary figures). We used the same datasets as the previous experiments. To train the networks faster (and avoid local minima) we employed a one-cycle learning rate schedule function that starts the learning rate from 0.001 before rising 0.5$\times\eta_{dset}$ ($\eta_{dset}$: dataset learning rate multiplier) in the first half of the training session and then dropping down to 0.001 again towards the end of the training. We used data augmentation in the Fashion, CIFAR-10 and Faces datasets. Table \ref{simp:net} shows the remaining parameters.

The comparison plots of the validation performances are shown in Figure \ref{fig:comp:res}. By inspecting the plots, we observed clear performance gains from adaptive 5$\times$5 and 7$\times$7 kernels in the MNIST-CLUT, CIFAR-10, and Fashion datasets. It was difficult to identify the best performers in the MNIST and Faces tests from the accuracy plots. Table \ref{tab:comp:res} demonstrates that the peak performances were those of the adaptive convolution (ACONV), with the exception of the Faces dataset. The mean peak validation accuracy differences in MNIST-CLUT, CIFAR-10, and Fashion were statistically significant. Further inspection of the Faces dataset results (see also Figure \ref{fig:comp:res}e), revealed that the highest mean peak accuracy was achieved by the ACONV $5\times5$ network (96.28\%); however, all comparisons were made at the kernel size which achieved the maximum peak accuracy, which was (CONV) $3\times3$ in this case. 

The active convolution model by Jeon \cite{jeon_2017} was also tested on the CIFAR-10 dataset, and test accuracy was reported as 92.46\% (single value) for an active convolution ResNet of 5 blocks and 32 layers. In contrast, we used 3 blocks and 20 layers which produced a maximum accuracy of $92.68\%$ and mean of $92.21\%$.

To explain the differences between the adaptive and ordinary convolutions, we computed the deep Taylor \cite{montavon2017,alber2018innvestigate} decompositions for three examples selected from the MNIST, Fashion and CIFAR-10 validation sets. The decompositions depicted in Figure \ref{fig:comp:res:f} represent the relevancy of individual input pixels back-traced from the network predictions. The heatmaps computed for CONV 3$\times$3 and CONV 7$\times$7 networks show that the larger kernel size network caused smoother and fuzzier relevancy regions. On the flip side, ACONV 7$\times$7 input heatmaps were smoother than CONV 3$\times$3 but sharper than CONV 7$\times$7. Moreover, the CIFAR-10 heatmaps revealed that ACONV 7$\times$7 input representation and focus was even more precise than CONV 3$\times$3.

\begin{figure*}
	\centering     
	\subfloat[MNIST]{\label{fig:comp:res:a}\includegraphics[width=0.5\linewidth]{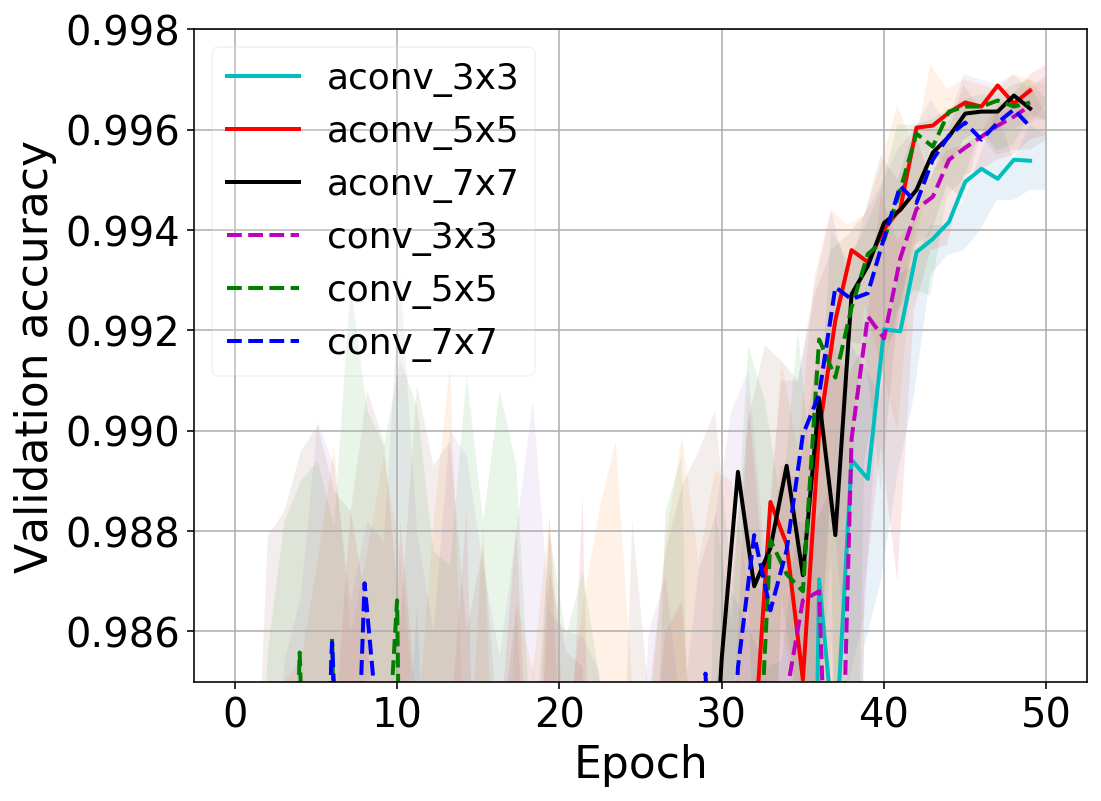}}
	\subfloat[MNIST-CLUT]{\label{fig:comp:res:b}\includegraphics[width=0.5\linewidth]{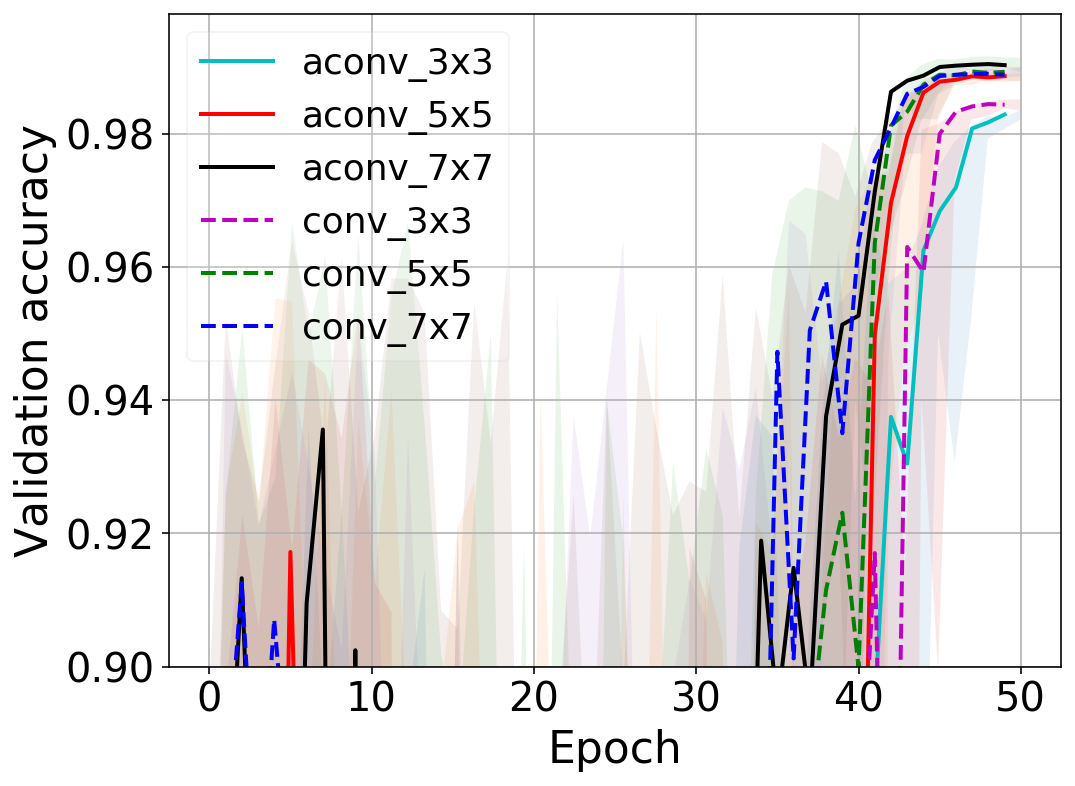}}\\
	\subfloat[CIFAR10]{\label{fig:comp:res:c}\includegraphics[width=0.5\linewidth]{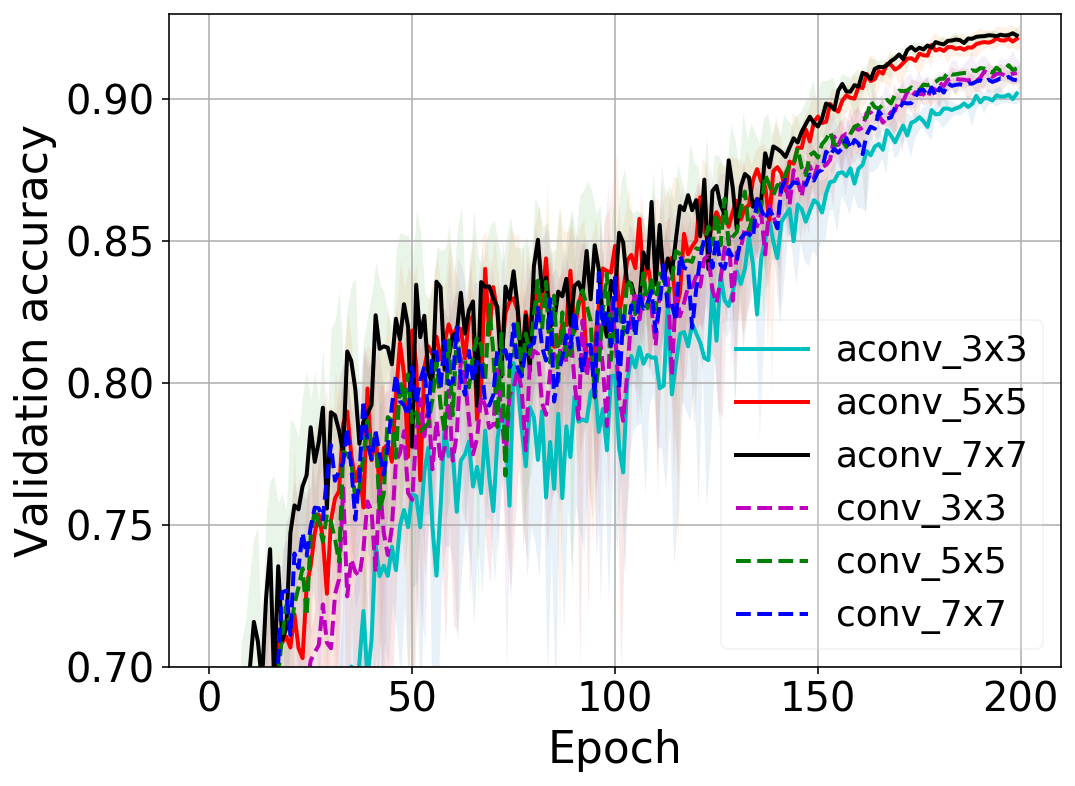}}
	\subfloat[FASHION]{\label{fig:comp:res:d}\includegraphics[width=0.5\linewidth]{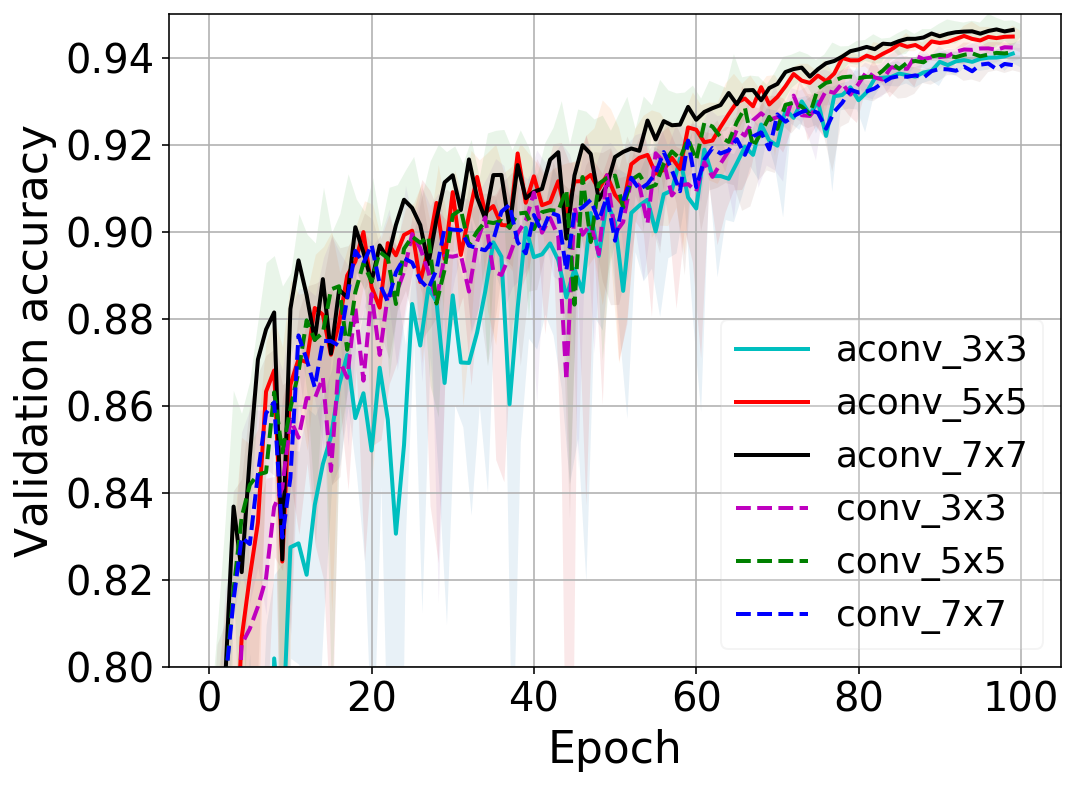}}\\
	\subfloat[FACES]{\label{fig:comp:res:e}\includegraphics[width=0.5\linewidth]{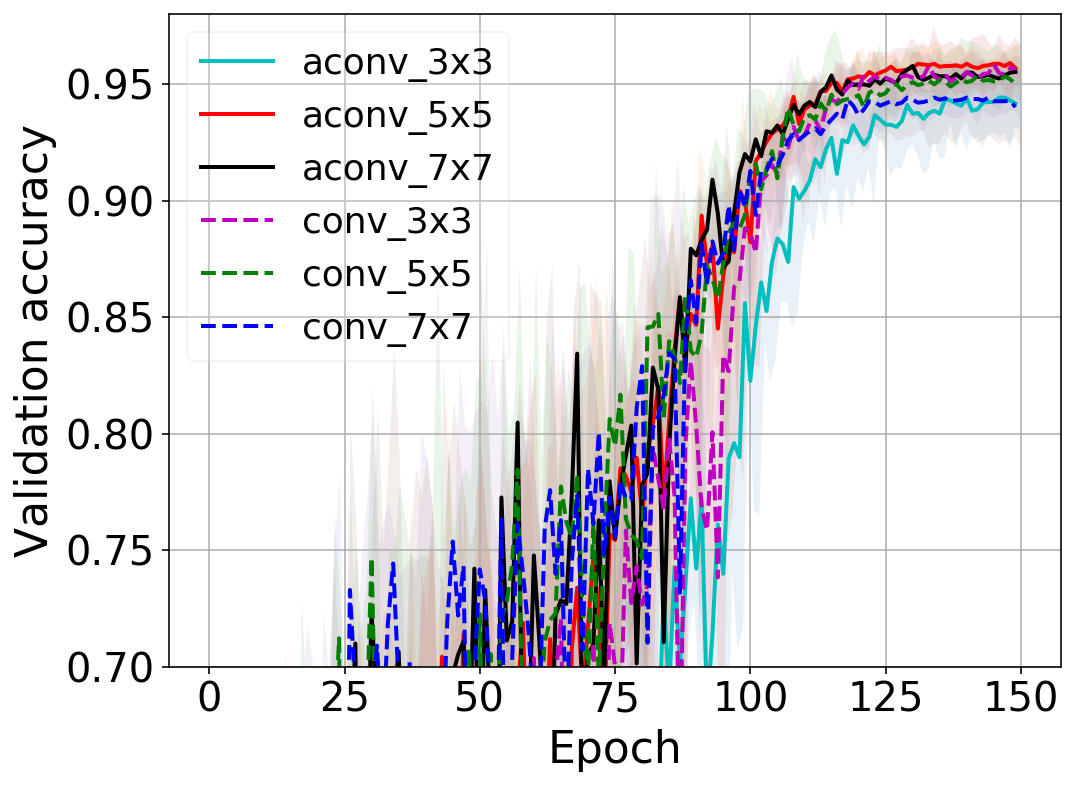}}
	\hfil
	\subfloat[Deep Taylor Decompositions]{\label{fig:comp:res:f}\includegraphics[width=0.4\linewidth]{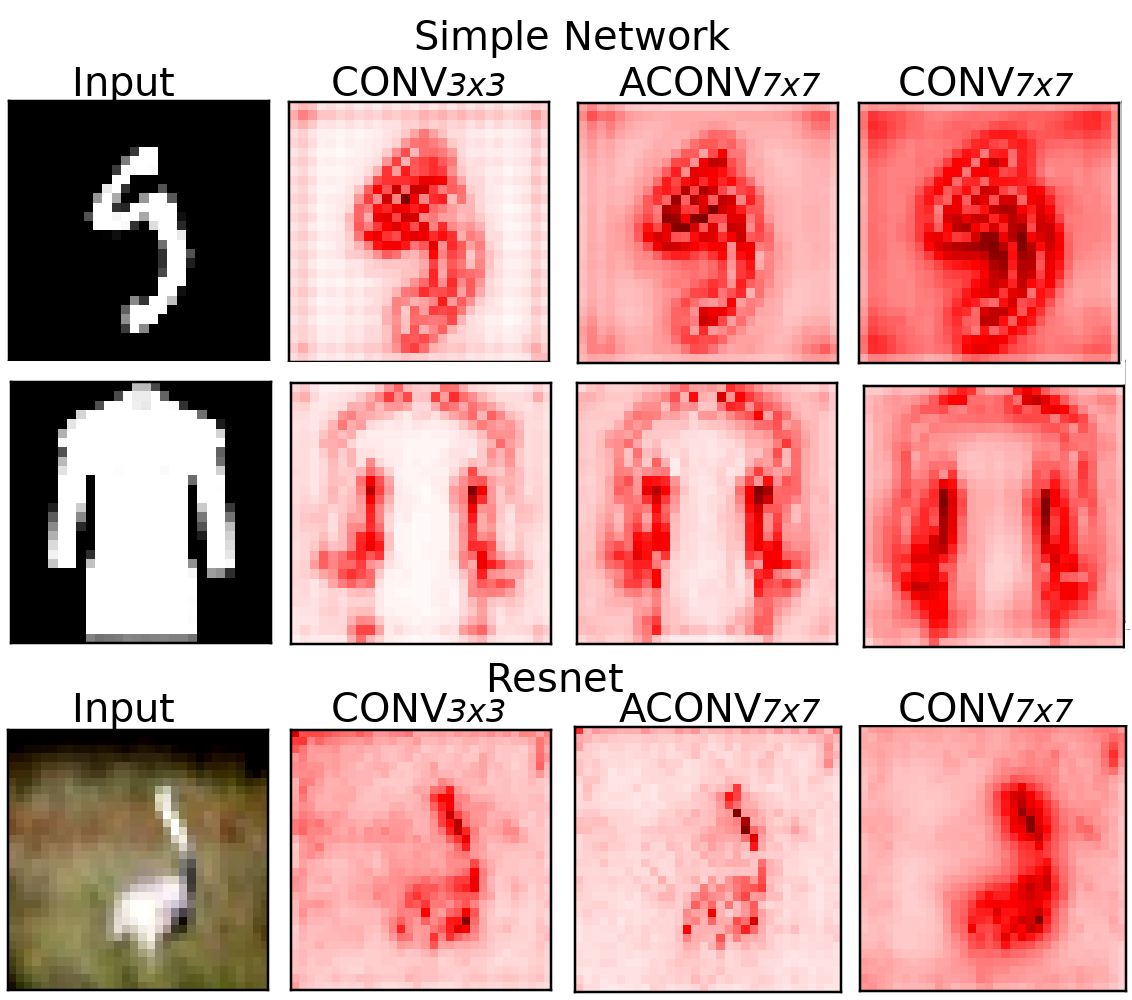}}
	\caption{Experiment 3: Resnet Comparisons. a-e) Validation Accuracy Plots. f) Deep Taylor Decompositions. }
	\label{fig:comp:res}
\end{figure*}

\subsection{Comparisons in U-Net for Segmentation}\label{sec:seg_unet}
The next experiment investigated the performance of the adaptive convolution in a U-shaped network architecture \cite{Ronneberger_15} which allows end-to-end image segmentation. The baseline code collected from Keras library \cite{keras} implements an efficient U-net architecture by using convolution, separable convolution, and deconvolution layers (see the supplementary for the configuration used). We took the first, ordinary convolution layer and replaced it with an adaptive kernel layer, then compared it against the original. The separable convolutions or deconvolution layers in the network were not replaced because they were not implemented in our adaptive kernel framework. The network was setup to learn tri-map segmented outputs of the images in the Oxford Pets-III dataset \cite{parkhi12a}, where the output classes are {pet, border, and background}. The dataset includes 7349 pictures of different dog and cat breeds together with their tri-map segmentation annotations. 

We resized all images to 128$\times$128 and then used the splits provided in the dataset to create training and validation sets of 3680 and 3669 instances. We used sparse categorical entropy loss and Adam optimizer using a batch size of 64 and a fixed learning rate of 0.01. We employed and compared 7$\times$7 adaptive kernels with 7$\times$7 and 3$\times$3 (baseline) ordinary kernels. All three networks were trained with five random initializations. Figure \ref{fig:com:unet:a} and \ref{fig:com:unet:b} plot the mean loss and validation accuracies over 75 training epochs. We observed that the networks started overfitting at shifted iterations ($>\approx$40 epochs). However, the mean training loss value reached by ACONV 7$\times$7 was lower than that of the ordinary kernels. Moreover, the mean peak validation accuracy achieved by ACONV 7$\times$7 (86.21\%$\pm$3e-3) was significantly higher than that of CONV 7$\times$7 (85.54\%$\pm$3e-3) with (p=0.0062), but only marginally higher than the accuracy reached by CONV 3$\times$3 (85.89\%$\pm$3e-3) (p=0.12).

Figures \ref{fig:com:unet:c} through \ref{fig:com:unet:q} compare the segmentation outputs qualitatively. While the output maps look very similar, the adaptive convolution layer network produced slightly more accurate border regions.

\begin{figure*}[t]
	\centering     
	\subfloat[Train Categorical Cross Entropy loss]{\label{fig:com:unet:a}\includegraphics[width=0.5\linewidth]{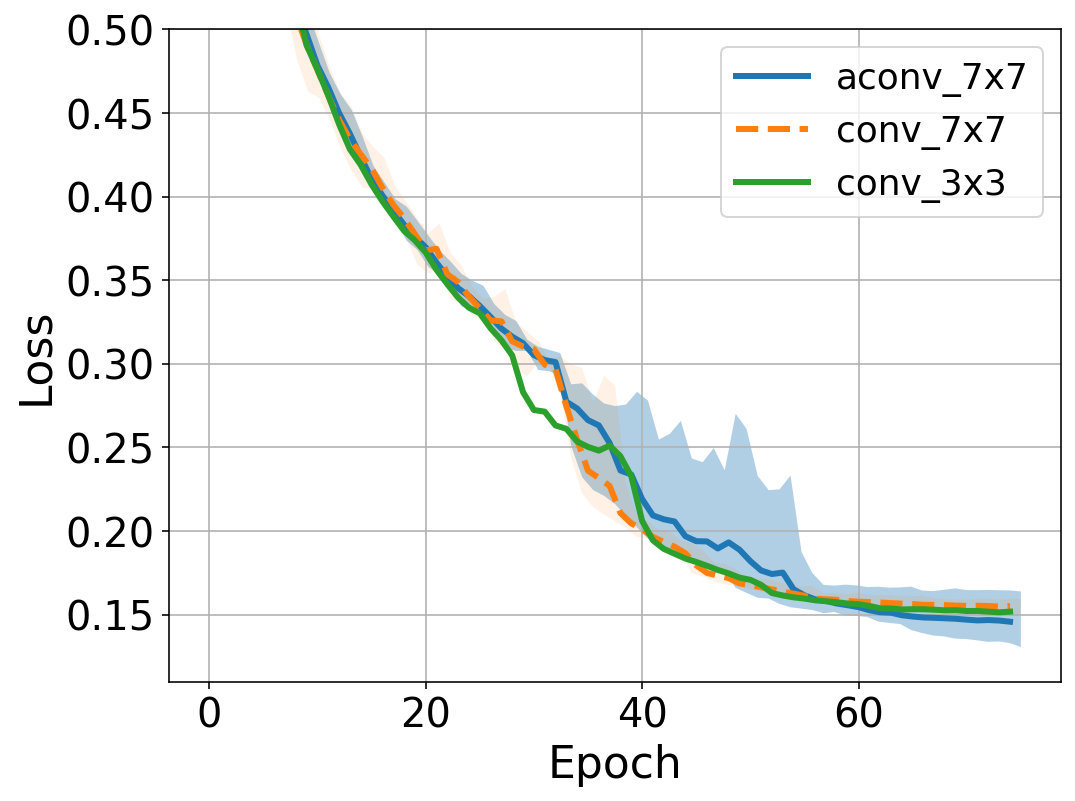}}		
	\subfloat[Validation Accuracy]{\label{fig:com:unet:b}\includegraphics[width=0.5\linewidth]{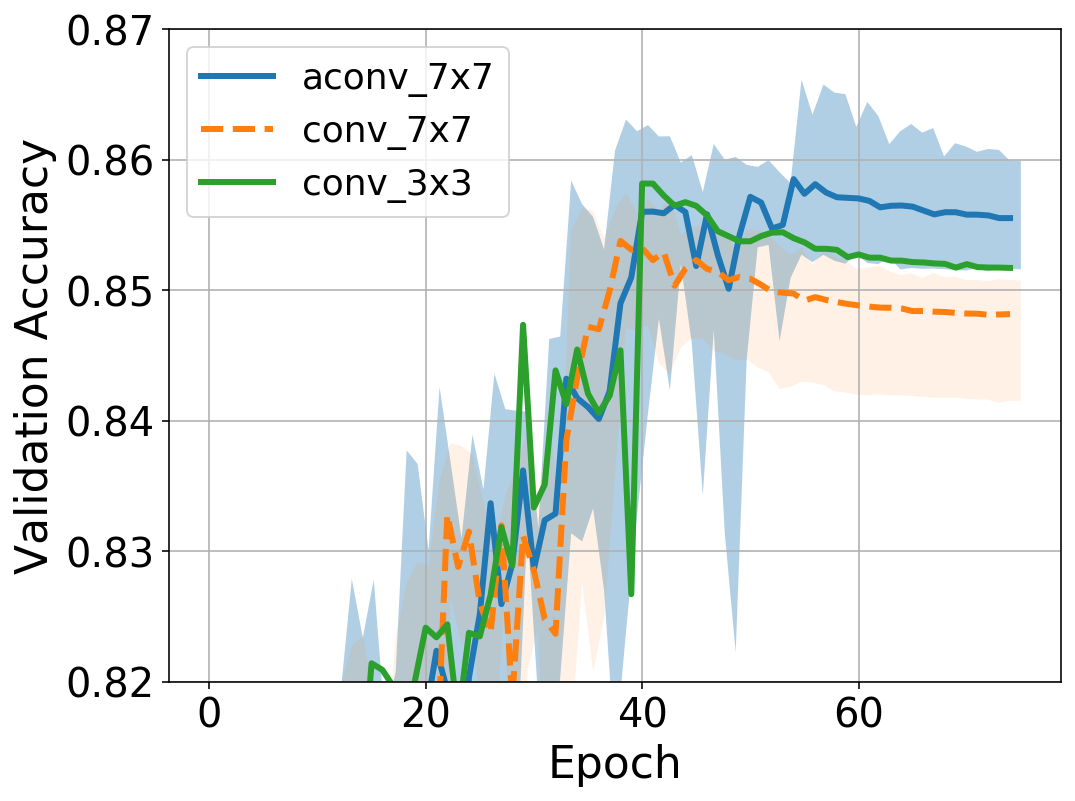}}
	
	\subfloat[Input]{\label{fig:com:unet:c}\includegraphics[width=0.2\linewidth]{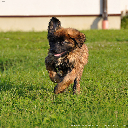}}
	\subfloat[True mask]{\label{fig:com:unet:d}\includegraphics[width=0.2\linewidth]{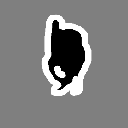}}
	\subfloat[ACONV 7$\times$7]{\label{fig:com:unet:e}\includegraphics[width=0.2\linewidth]{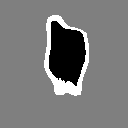}}
	\subfloat[CONV 7$\times$7]{\label{fig:com:unet:f}\includegraphics[width=0.2\linewidth]{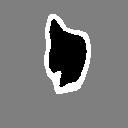}}
	\subfloat[CONV 3$\times$3]{\label{fig:com:unet:g}\includegraphics[width=0.2\linewidth]{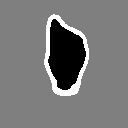}}\\
	
	\subfloat[Input]{\label{fig:com:unet:h}\includegraphics[width=0.2\linewidth]{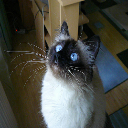}}
	\subfloat[True mask]{\label{fig:com:unet:i}\includegraphics[width=0.2\linewidth]{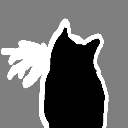}}
	\subfloat[ACONV 7$\times$7]{\label{fig:com:unet:j}\includegraphics[width=0.2\linewidth]{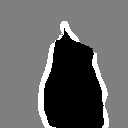}}
	\subfloat[CONV 7$\times$7]{\label{fig:com:unet:k}\includegraphics[width=0.2\linewidth]{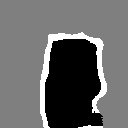}}
	\subfloat[CONV 3$\times$3]{\label{fig:com:unet:l}\includegraphics[width=0.2\linewidth]{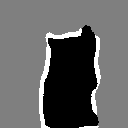}}\\	
	\subfloat[Input]{\label{fig5:a}\includegraphics[width=0.2\linewidth]{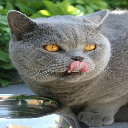}}
	\subfloat[True mask]{\label{fig:com:unet:m}\includegraphics[width=0.2\linewidth]{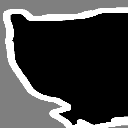}}
	\subfloat[ACONV 7$\times$7]{\label{fig:com:unet:n}\includegraphics[width=0.2\linewidth]{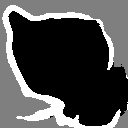}}
	\subfloat[CONV 7$\times$7]{\label{fig:com:unet:p}\includegraphics[width=0.2\linewidth]{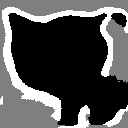}}
	\subfloat[CONV 3$\times$3]{\label{fig:com:unet:q}\includegraphics[width=0.2\linewidth]{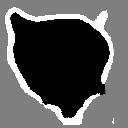}}\\
	\caption{U-Net segmentation comparisons. a-b) Training and validation accuracy plots of U-net with adaptive convolution (ACONV 7$\times$7) and ordinary convolutions (CONV 7$\times$7 and CONV 3$\times$3). c-q) Input, true mask, and predictions of the networks for three different input images from the validation set.}
	\label{fig:com:unet}
\end{figure*}

\subsection{Time Complexity}
The additional complexity of the adaptive kernel is due to the calculation of the envelope function which depends on the kernel size only; it is independent of the input width and height. During training, the envelope function must be calculated for each batch using the current aperture values of individual filters $\sigma_u^q$ while during back-propagation, an extra gradient is calculated for $\sigma_u^q$. In the MNIST training, we recorded the following forward+backward mean batch (128) step-times for different kernel sizes of ACONV, where the time for the ordinary CONV kernel of the same size is given in parentheses: $\lbrace3\times3$: 215us (166us), 5$\times$5: 238us (180us), 7$\times$7: 321us (282us), 9$\times$9: 340us (285us)$\rbrace$, on a laptop equipped with i7-8565U and NVIDIA 1650 GPU. Therefore, the adaptive network was $\approx$1.2 to $\approx$1.3 times slower than an ordinary kernel of the same size in training.  However, we must note that the current implementation was not optimized for speed at all. In addition, in run-time the overhead of the envelope can be removed by using the learned product kernels ($U\circ W$).

\section{Discussions}
The experiments demonstrated the feasibility of the proposed adaptive kernel model. First, the adaptive kernel was able to learn different image processing filters without encountering any difficulty. The learned kernel shapes demonstrated that the envelope and weights were able to co-adapt successfully during the training.

Second, the comparative tests in a simple convolutional network configuration demonstrated the learning and generalization performances on popular image classification datasets. In all datasets, the adaptive kernels provided significant but slight improvements in generalization performance, more than the potential gains that would be achieved by using the ordinary kernels of larger size.

In the ResNet architecture, the use of adaptive kernels resulted in better generalization performance compared to the ordinary kernels in all datasets except Faces, in which the maximum peak validation accuracy was in favor of the ordinary kernel of size $3\times3$. However, we noted that the mean accuracy of the $5\times5$ adaptive kernel was higher. Therefore, we recommend employing $5\times5$ or $7\times7$ adaptive kernels for potential performance gains in ResNet architectures. An additional insight gained from our experiments was that, in contrast to the widely accepted usage of $3\times3$ ordinary convolution kernels, the larger kernels may work better in ResNet for some datasets.

The kernels learned in the Fashion dataset (Figure \ref{fig:simp:comp:f}) verified that the adaptive layers were able to create varying-sized kernels. Furthermore, the deep Taylor decomposition analysis of the compared Resnets (Figure \ref{fig:comp:res:f}) displayed evidence for multi-scale representation computed by the adaptive convolution networks.

In a brief segmentation experiment, we tested the adaptive kernels in an efficient U-net architecture by replacing a single convolution layer, which resulted in an improved segmentation performance against the same (larger) size kernels; however, there was no significant gain compared to smaller 3$\times$3 kernel layer. However, the current state-of-the-art image segmentation methods use more complex architectures and architecture-search algorithms \cite{zhang2020resnest,howard2019searching}. Therefore, it would be appropriate to study adaptive kernels for segmentation in a dedicated study.

In summary, the experiments demonstrated that the adaptive kernel model is an effective alternative to the ordinary convolution kernel. It can create varying-sized kernels in a single layer. It is less prone to overfitting than an ordinary large convolution kernel (5$\times$5, 7$\times$7, 9$\times$9) while providing better or comparable performance to the widely employed 3$\times$3 ordinary kernel.

\section{Conclusion}
In conclusion, we here propose an adaptive convolution kernel which is able to learn its size by training with backpropagation. The new model is standalone, modular and compatible with existing Keras and Tensorflow backends. Hence, one can easily import and attach the proposed adaptive layer into a network and train it with any stride or dilation factor. The single additional requirement is to apply a clip (callback) to the aperture parameter ($\sigma_u$) to keep it above a minimum positive value during training iterations.

There were some limitations to our study, which may be addressed by future work. The current state-of-the-art networks require large resources to set up, tune and optimize on larger datasets. It will be interesting to observe the learning performance of the adaptive kernels in a state-of-the-art network on one of the large datasets. Next, it will be necessary to set up a dedicated segmentation study to compare the adaptive kernel model against the ordinary kernels and other adaptive methods such as deformable or active convolution models.


\section*{Funding}
This work was supported by TUBITAK 1001 programme (no: 118E722), Isik University BAP programme (no: 16A202), and NVIDIA hardware donation of a Tesla K40 GPU unit.

\section*{Acknowledgments}
Earlier implementation and experiments were conducted by İ. Çam; a draft was prepared by İ. Çam, F. B. Tek coded the kernels again in Keras/Tensorflow performed the current experiments, and wrote the current paper. D. Karlı contributed to the mathematical model and proofs. Thanks to Mert Mısırlıoğlu for helping with the ResNet experiment setup. 

\bibliographystyle{model1-num-names.bst}
\bibliography{focusv3}

\appendix 

\section{Derivation of mean of the variances of the weight derivatives}\label{apx_w_init}
We assume that the inputs and weights are i.i.d (independent and identically distributed) and the expected values are zero, i.e.\ $\mathbb{E}({x}_{(i,j)})=0$ and $\mathbb{E}(w_{k_l})=0$. Let us recall the expression for the mean of variances,
\begin{equation}
\MVar \bigg(\frac{\partial E}{\partial W} \bigg)  =  \frac{1}{n^2} \sum_{k,l}^{\nn}\Var \bigg({\frac{\partial E}{\partial w_{k,l}}}\bigg)
\end{equation}
Here, the expected value of the derivative is also zero $\mathbb{E} \bigg(\frac{\partial E}{\partial w_{k,l}} \bigg) = 0 $ by the independence of variables and $\mathbb{E}({x}_{i,j})=0$. Hence:
\begin{equation}
\Var\big( \frac{\partial E}{\partial w_{k,l}} \big) = \mathbb{E}\bigg( \big[\frac{\partial E}{\partial w_{k,l}} \big]^2\bigg) = 
\mathbb{E} \Bigg((u_{k,l}^2) \bigg[\sum_{i,j}^{M,N}\frac{\partial E}{\partial {o}_{i,j}} x_{i+j,k+l}\bigg]^2\Bigg)
\end{equation}

By independence, the last line equals: 
\begin{multline}
=\mathbb{E}(u_{k,l}^2) \Bigg[\sum_{i,j}^{M,N} \mathbb{E} \bigg( \bigg[\frac{\partial E}{\partial {o}_{i,j}} x_{i+j,k+l}\bigg]^2 \bigg) \\+ 
2\;\sum_{i,j}^{M,N} \; \sum_{p,r}^{i-1,j-1} \mathbb{E} \bigg( \frac{\partial E}{\partial {o}_{i,j}}
\frac{\partial E}{\partial {o}_{p,r}} x_{i+k,j+l} \; x_{i+p,k+r} \bigg)
\Bigg]
\end{multline}
The second term is zero since $\mathbb{E}({x}_{i,j})=0$ and independence of input $x_{i,j}$ from the other variables. Then, 
\begin{equation}
\Var\big( \frac{\partial E}{\partial w_{k,l}} \big) =
\mathbb{E}(u_{k,l}^2) \; \sum_{i,j}^{M,N} \mathbb{E}(x_{i+k,j+l}^2) \; \mathbb{E} \Big( \big[\frac{\partial E}{\partial {o}_{i,j}}\big]^2 \Big)
\end{equation}
Since $x_{i,j}$ are i.i.d with variance $\sigma_x^2$ and expectation zero, we can write 
\begin{equation}
\Var\big( \frac{\partial E}{\partial w_{k,l}} \big) =
\sigma_x^2 \; \mathbb{E}(u_{k,l}^2) \; \sum_{i,j}^{M,N} \mathbb{E} \Big( \big[\frac{\partial E}{\partial {o}_{i,j}}\big]^2 \Big)
\end{equation}
Then the mean of variances is as follows:
\begin{equation}
\MVar\big( \frac{\partial E}{\partial W} \big) =  \sigma_x^2 \bigg[ \frac{1}{n^2} \sum_{k,l}^{\nn}\mathbb{E}(u_{k,l}^2) \bigg] \sum_{i,j}^{M,N} \mathbb{E} \bigg[ \big(\frac{\partial E}{\partial {o}_{i,j}} \big)^2 \bigg]
\end{equation}

\end{document}